\documentclass{article}

\usepackage{arxiv}

\usepackage{epstopdf}
\usepackage{amsthm}
\usepackage{url}
\usepackage[utf8]{inputenc} 
\usepackage[T1]{fontenc}    
\usepackage{booktabs}       
\usepackage{amsfonts}       
\usepackage{nicefrac}       
\usepackage{microtype}      
\usepackage{mathtools}
\usepackage{amssymb}
\usepackage[shortlabels]{enumitem}
\usepackage{xcolor,colortbl}
\usepackage{comment}
\usepackage{subcaption}
\usepackage{graphicx}
\usepackage{makecell}
\usepackage{multirow}
\usepackage{float}
\usepackage{caption}
\usepackage[flushleft]{threeparttable}
\usepackage{algorithm}
\usepackage{algpseudocode}
\usepackage{lmodern}

\usepackage{hyperref}

\newtheorem{theorem}{Theorem}

\newtheorem{remark}{Remark}
\newtheorem{lemma}[theorem]{Lemma}

\newtheorem{assumption}{Assumption}

\numberwithin{theorem}{section}
\title{A Generalized McDiarmid Inequality: Non-Uniform Unbounded Differences and Sample Reuse}

\author{  Jeremiah Birrell\\
  Department of Mathematics\\
  Texas State University\\
  San Marcos, TX,  USA \\
  \texttt{jbirrell@txstate.edu} 
}
\begin{document}
\maketitle

\begin{abstract}
We develop a  generalization of McDiarmid's concentration inequality and use it to obtain a uniform law of large numbers result and error bounds on stochastic optimization problems which simultaneously incorporate several key features: 1) They capture  
the benefit of training sample reuse in algorithms that also employ easily sampled auxiliary random variables. 2) They allow for  one-component mean-difference bounding functions which are both unbounded in the component of interest and have  non-uniform dependence on another  component. 3) When the data distribution has low-dimensional support, our methods are able to leverage this to mitigate the curse of dimensionality, even when the auxiliary random variables do not have low-dimensional support.  We demonstrate the utility of our method by deriving concentration inequalities for  score-based generative modeling. This application naturally leverages the key features of our new tools.   The ability to handle non-uniform mean-difference bounding functions allows us to treat unbounded score models and data distributions with unbounded support.   DSM utilizes auxiliary Gaussian random variables in conjunction with the training samples; accounting for training sample reuse   results in tighter error bounds that mitigate the curse of dimensionality.
\end{abstract}

\section{Introduction}\label{sec:intro}
In this work we derive a novel generalization of McDiarmid's bounded differences concentration inequality  \cite{mcdiarmid1989method} which allows for  non-uniform one-component difference bounds and use this to derive  statistical error bounds on uniform law of large numbers (ULLN) and  stochastic optimization problems. Our focus is specifically on problems which inherently utilize two types of random variables:  samples, $X$, from the distribution of interest and easily sampled auxiliary random variables, $Y$, which are independent of $X$. Such situations occur  in the training of score-based generative models via denoising score matching (DSM) \cite{songscore},   which utilizes auxiliary Gaussian random variables.

The ability to generate  samples $Y\sim P_Y$ at will  motivates us to consider the statistical error   of objectives which use $m$-times more $Y$-samples than $X$-samples (where $X\sim P_X$), as in the following:
\begin{flalign}\label{eq:ULLN_intro}
\text{{\bf ULLN:}} &\qquad\qquad\qquad\qquad \sup_{\theta\in\Theta}\left\{\frac{1}{nm}\sum_{i=1}^n \sum_{j=1}^m g_\theta(X_i,Y_{i,j})-E_{P_X\times P_Y}[g_\theta]\right\}\,,&&
\end{flalign}
\begin{flalign}\label{eq:stoch_opt_intro}
\text{{\bf Stochastic Optimization:}} \quad \qquad &E_{P_X\times P_Y}[g_{\theta^*_{n,m}}]-\inf_{\theta\in\Theta} E_{ P_X\times P_Y}[g_\theta ]\,,\\
 & \theta^*_{n,m}\in\operatorname*{arg\,min}_{\theta\in\Theta}\frac{1}{nm}\sum_{i=1}^n\sum_{j=1}^m g_\theta(X_i,Y_{i,j})\,.&&\notag
\end{flalign}
The principal novelty of our methods lies in their ability to simultaneously exhibit the following properties. Our results capture the benefit of using many $Y$-samples per $X$-sample; specifically,  for large $m$ our  bounds  approach  those obtained on the corresponding  problem with the $Y$-dependence integrated out, i.e., bounds on
\begin{align}\label{eq:ULLN_intro_Y_integrated}
  \sup_{\theta\in\Theta}\left\{\frac{1}{n}\sum_{i=1}^n   E_{P_Y}[g_\theta(X_i,\cdot)]-E_{P_X\times P_Y}[g_\theta]\right\}
\end{align}
and on
\begin{flalign}\label{eq:stoch_opt_intro_Y_integrated}
&E_{P_X\times P_Y}[g_{\theta^*_{n}}]-\inf_{\theta\in\Theta} E_{ P_X\times P_Y}[g_\theta ]\,,\\
&\theta^*_{n}\in\operatorname*{arg\,min}_{\theta\in\Theta}\frac{1}{n}\sum_{i=1}^nE_{P_Y}[ g_\theta(X_i,\cdot)]\notag\,,
\end{flalign}
respectively.  For large $m$, our bounds will be controlled by the Rademacher complexity of the function class $x\mapsto E_{P_Y}[ g_\theta(x,\cdot)]$, $\theta\in\Theta$, and hence they can take advantage of an $X$-distribution that has low-dimensional support, even when the $Y$-distribution does not. Moreover, our methods are capable of simultaneously handling  unbounded objective functions and distributions with unbounded support.   These features distinguish the present work from
\cite{doi:10.1137/25M1748664}, which also studied the  benefits of sample reuse in problems of the form \eqref{eq:ULLN_intro} and \eqref{eq:stoch_opt_intro}, but with more restricted tools
which did not allow for applications to, e.g., DSM for data distributions with unbounded support, and which were not capable of addressing the curse of dimensionality.

The main difficulty in achieving the aforementioned goals lies in the interaction between  a single $X_i$-sample  and  multiple distinct $Y_{i,j}$-samples  in \eqref{eq:ULLN_intro} and \eqref{eq:stoch_opt_intro}. The new variant of   McDiarmid's inequality that we develop is specifically adapted to this setting.  It produces concentration inequalities which depend on moment generating function (MGF) bounds in $Y$ which are non-uniform in   $X$. To facilitate explicit computations, we also develop a   technique for   computing such MGF bounds  via Orlicz norms.

\subsection{Summary of the Main Results}\label{sec:summary}
The following is a more detailed summary of our main results.   
\begin{enumerate}
\item In Theorem \ref{thm:gen_McDiarmid_non_unif_MGF} we derive a novel generalization of McDiarmid's inequality which allows for one-component mean-difference bounds (and corresponding MGF bounds) that are non-uniform in another variable.
\item In Lemma \ref{lemma:Orlicz_cosh_MGF_bound_sum} we provide a   tool for obtaining explicit MGF bounds which appropriately handle the non-uniform dependence on other variables. This is achieved by breaking the bounding function into multiple terms, each of which can be paired with an appropriate Orlicz norm based on its tail behavior.
\item  In Theorem  \ref{thm:ULLN_bound_non_unif_MGF_Phi1_Phi2_case} we show how  Theorem \ref{thm:gen_McDiarmid_non_unif_MGF} and  Lemma \ref{lemma:Orlicz_cosh_MGF_bound_sum} can be combined to yield an explicit concentration inequality when the non-uniform bounding functions are appropriately composed of sub-Gaussian and sub-exponential terms.

\item In Section \ref{sec:ULLN_stoch_opt} we use  Theorem  \ref{thm:ULLN_bound_non_unif_MGF_Phi1_Phi2_case} to obtain  concentration inequalities for ULLN and stochastic optimization problems.  In cases with sample reuse, these results capture the desired large $m$ behavior that was discussed qualitatively above; see Remarks \ref{remark:ULLNlarge_m_limit} and \ref{remark:StochOptlarge_m_limit}.
\item In Lemmas \ref{lemma:ULLN_mean_decomp}\,-\,\ref{lemma:mean_bound} and Theorem \ref{thm:mean_bound_smooth} we obtain mean ULLN bounds (which are required by the corresponding concentration inequalities) that capture the benefits of sample reuse and are capable of addressing the curse of dimensionality; see Remark \ref{remark:Rademacher_large_m}. 
\item Finally, in Section \ref{sec:DSM} we show that, by using  these new multi-purpose tools, it is straightforward to derive concentration inequalities for score-based generative models which capture the benefits of sample reuse (many Gaussian samples per training sample) while simultaneously allowing for an unbounded score model and a data distribution with unbounded support. Moreover, the  leading behavior of our bounds is governed by $d_*$, the intrinsic dimension of support of the $X$-distribution, thus addressing the curse of dimensionality.
\end{enumerate}

\subsection{Related Works}\label{sec:related_works}

The generalization of McDiarmid's original bounded differences concentration inequality  \cite{mcdiarmid1989method}  that we present in this paper can be viewed as an extension of the techniques developed in \cite{kontorovich2014concentration} and \cite{doi:10.1137/25M1748664}, both of which relaxed the uniform boundedness condition.  In \cite{kontorovich2014concentration} the theory was extended  to one-component Lipschitz bounds on a space with finite sub-Gaussian diameter.
  In  \cite{doi:10.1137/25M1748664} the theory was  further extended   to non-Lipschitz one-component difference bounds  which are non-uniform in subsequent variables in the chosen ordering of the components (with the dependence on subsequent variables integrated out).  Like the aforementioned works, the  method we develop in this paper is  based on the method of martingale differences.  An alternative approach, based on the entropy method,   was used by \cite{maurer2021concentration} to relax the bounded differences assumption, allowing them to treat sub-Gaussian and sub-exponential one-component difference bounds.    More loosely related generalizations of McDiarmid's inequality include  \cite{combes2024extension}, which considered bounded differences on a high probability set, and \cite{warnke2016method}, which obtained tighter bounds in cases where the worst-case one-component differences are large (though still uniformly bounded) but the typical differences are small.   We also note a different stream of work studying  generalizations to dependent random variables \cite{kontorovich2017concentration,NEURIPS2019_658bbbde,zhang2022janson};  here, our focus is solely on problems involving independent random variables.

 The key innovation of the present work is a generalization   to one-component difference bounds that, in addition to being unbounded, are   non-uniform both in subsequent variables and  in particular earlier variables in the chosen variable ordering; this allows for applications to problems of the form  \eqref{eq:ULLN_intro} and \eqref{eq:stoch_opt_intro} where $g_\theta(x,y)$ has $x$-difference bounds that are non-uniform and  unbounded in $y$ and also has $y$-difference bounds that are non-uniform and unbounded in $x$.  The aforementioned previous approaches were not capable of handling this type of behavior.  We do note that the constants in the classical  McDiarmid inequality  are better than those in the present work, as well as the other extensions to unbounded differences cited here, when using  uniform bounds on one-component differences. This fact was previously discussed in    \cite{kontorovich2014concentration} and \cite{doi:10.1137/25M1748664}; the same discussion applies here.  The key advantage of our method  lies in the ability to treat a more general class of unbounded non-uniform one-component differences than could previously be handled. Table \ref{table:comparison} summarizes the key differences between our method, the classical McDiarmid inequality, and other extensions which apply to unbounded one-component differences.

\begin{table}[h!]
\centering 
\caption{Comparison of our Theorem \ref{thm:gen_McDiarmid_non_unif_MGF} with the generalized McDiarmid result from \cite{doi:10.1137/25M1748664},  the entropy-based approach of \cite{maurer2021concentration}, the sub-Gaussian diameter method from \cite{kontorovich2014concentration}, and McDiarmid's inequality  \cite{mcdiarmid1989method}.}\label{table:comparison}
\setlength{\extrarowheight}{6pt} \footnotesize
\begin{tabular}{|l|c|c|c|c|}
\hline
\makecell{Method} & \makecell{Tail behaviors\\covered} & \makecell{Allows non-uniform\\dependence on\\earlier variables} & \makecell{Allows non-uniform\\dependence on\\later variables} & \makecell{Tightest\\constants} \\
\hline\hline
\makecell{Theorem \ref{thm:gen_McDiarmid_non_unif_MGF}} & \makecell{general MGF or\\ Orlicz-norm bound} & yes & yes & no \\
\hline
 \makecell{Generalized McDiarmid\\ \cite{doi:10.1137/25M1748664}} &  \makecell{general MGF or\\ Orlicz-norm bound} & no & yes & no \\
\hline
\makecell{Entropy method\\ \cite{maurer2021concentration}} &  \makecell{sub-Gaussian or\\ sub-exponential} & no & no & no \\
\hline
 \makecell{Sub-Gaussian diameter\\ \cite{kontorovich2014concentration}} & sub-Gaussian & no & no & no \\
\hline
 \makecell{McDiarmid's inequality\\ \cite{mcdiarmid1989method}} & compact support & no & no & yes \\
\hline
\end{tabular}
\end{table}

To demonstrate the utility of our results, in Section \ref{sec:DSM} we apply them to   score-based generative models  \cite{songscore}; specifically, we obtain statistical error bounds, in the form of concentration inequalities, on the denoising score matching objective.   A number of works have derived   convergence results or error bounds for score-based generative models under various assumptions \cite{chen2023improved,pedrotti2023improved,bentonnearly,li2024towards,li2024sharp,potaptchik2024linear,chen2024equivariant,mimikos2024score,gao2025wasserstein,chakraborty2026generalization}. Closest to our application to DSM are: 1) the  $L^1$ error bounds and bounds on the estimation error in the form of expected total variation and $1$-Wasserstein distance in \cite{oko2023diffusion}, for data distributions with compact support; 2) the more general Wasserstein bounds from \cite{chakraborty2026generalization}, which   address the curse of dimensionality and apply to data distributions with unbounded support; 3) the use of a weaker generalization of McDiarmid's theorem in \cite{doi:10.1137/25M1748664}  to obtain concentration inequalities for DSM that applied to data distributions with compact support.  We do note that \cite{oko2023diffusion,chakraborty2026generalization} present more  holistic analyses of the multiple sources of error in score matching diffusion models.   The focus of the present work is the development of multi-purpose tools for obtaining concentration inequalities for ULLN and stochastic optimization problems of the form \eqref{eq:ULLN_intro} and \eqref{eq:stoch_opt_intro}, which are of independent interest beyond DSM, rather than a targeted analysis of all aspects of DSM; hence  we only consider the statistical error of learning the score model via DSM.

\section{A Generalized McDiarmid Inequality via  Non-Uniform MGF Bounds}\label{sec:generalized_mcdiarmid_var_reuse_mean_bound_nonunif_method}
In this section we develop our new variant of McDiarmid's inequality which allows for non-uniform mean-difference bounds.  We will focus on one specific form of non-uniformity, where the variables can be ordered and partitioned into groups such that the earlier-variable non-uniformity is all with respect to the first variable of each grouping, while any later-variable non-uniformity is integrated out.   This  is the relevant setting for problems with dual sample types and sample reuse, as in \eqref{eq:ULLN_intro} and \eqref{eq:stoch_opt_intro}.
\begin{assumption}\label{asump:gen_McDiarmid_non_unif_MGF}
Let $n\in\mathbb{Z}^+$, $m_i\in \mathbb{Z}^+$ for $i=1,\dots,n$. Suppose we have measurable spaces $ \mathcal{X}_{i}$, $i=1,\dots,n$, $\mathcal{Y}_{i,j}$, $i=1,\dots,n$, $j=1,\dots,m_i$; define $\mathcal{Z}_i\coloneqq \mathcal{X}_i\times \prod_{j=1}^{m_i} \mathcal{Y}_{i,j}$   and $\mathcal{Z}\coloneqq\prod_{i=1}^n\mathcal{Z}_i$  (with the product $\sigma$-algebras).

Suppose we have a probability space $(\Omega,\mathbb{P})$ on which we have $\mathcal{X}_i$-valued random variables $X_i$, $i=1,\dots,n$, and $\mathcal{Y}_{i,j}$-valued random variables $Y_{i,j}$, $i=1,\dots,n$, $j=1,\dots,m_i$, that are all mutually independent. We let $P_{X_i}$ denote the  distribution of $X_i$ and $P_{Y_{i,j}}$ denote the distribution of $Y_{i,j}$. We will also make use of the definitions $Y_i^{k:\ell}\coloneqq (Y_{i,k},Y_{i,k+1},\dots,Y_{i,\ell})$,  $Z_i\coloneqq (X_i,Y_{i,1},\dots,Y_{i,m_i})$, $Z^{k:\ell}\coloneqq (Z_k,Z_{k+1},\dots,Z_\ell)$ (the case where $k>\ell$ simply means this term is to be omitted) and $Z\coloneqq (Z_1,\dots,Z_n)$. We use similar lower-case notation for elements of the product spaces that correspond to the codomains of these random variables.

Suppose also that we have a measurable function $\phi:\mathcal{Z}\to\mathbb{R}$ which satisfies the following:
\begin{enumerate}
\item Integrability: $\phi(Z)\in L^1(\mathbb{P})$ and for  $i=1,\dots,n$, $\ell=0,\dots,m_i$, we have
\begin{align}
&\phi(z^{1:i-1},x_i,y_i^{1:\ell},Y_i^{\ell+1:m_i},Z^{i+1:n})\in L^1(\mathbb{P})
\end{align}
  for all $z^{1:i-1},x_i, y_i^{1:\ell}$.
\item One-component  mean-difference bounds for $\mathcal{X}$:  For  $i=1,\dots,n$, we have measurable $h_{\mathcal{X}_i}:\mathcal{X}_i\times \mathcal{X}_i\to [0,\infty]$ such that
\begin{align}\label{eq:X_mean_diff_bound}
\left|\mathbb{E}[\phi(z^{1:i-1},x_i,Y_i^{1:m_i},Z^{i+1:n})]-\mathbb{E}[\phi(z^{1:i-1},\tilde{x}_i,Y_i^{1:m_i},Z^{i+1:n})]\right|\leq h_{\mathcal{X}_i}(x_i,\tilde{x}_i)
\end{align}
for all $z^{1:i-1},x_i,\tilde{x}_i$.
 \item One-component  mean-difference bounds for $\mathcal{Y}$:  For all $i=1,\dots,n$, $j=1,\dots,m_i$, we have measurable $h_{\mathcal{Y}_{i,j}}:\mathcal{X}_i\times\mathcal{Y}_{i,j}\times \mathcal{Y}_{i,j}\to [0,\infty]$ such that
\begin{align}\label{eq:Y_mean_diff_bound}
&\left|\mathbb{E}[\phi(z^{1:i-1},x_i,y_i^{1:j},Y_i^{j+1:m_i},Z^{i+1:n})]-\mathbb{E}[\phi(z^{1:i-1},x_i,y_i^{1:j-1},\tilde{y}_{i,j},Y_i^{j+1:m_i},Z^{i+1:n})]\right|\\
\leq&  h_{\mathcal{Y}_{i,j}}(x_i,y_{i,j},\tilde{y}_{i,j})\notag
\end{align}
for all $z^{1:i-1},x_i,y_i^{1:j},\tilde{y}_{i,j}$.
\end{enumerate}
\end{assumption}
\begin{remark}
We emphasize that our results are still applicable when there is no sample reuse, i.e., when $m_i=1$ for all $i$, though that is not our   focus.
\end{remark}

We now proceed to derive  our generalization of McDiarmid's inequality. Here and in the following we will use $P^k$ to denote the $k$-fold product of a probability measure $P$.
\begin{theorem}\label{thm:gen_McDiarmid_non_unif_MGF}
Under Assumption \ref{asump:gen_McDiarmid_non_unif_MGF}, suppose we also have the following MGF  bounds:
\begin{enumerate}
\item $\mathcal{X}$-MGF bounds: We have $\xi_{\mathcal{X}_i}:[0,\infty)\to[0,\infty]$, $i=1,\dots,n$, such that
\begin{align}\label{eq:gen_McDiarmid_non_unif_X_MGF_assump}
\int \cosh(\lambda  h_{\mathcal{X}_i}) dP_{{X}_i}^2\leq \exp(\xi_{\mathcal{X}_i}(|\lambda|)) \,\,\,\text{  for all $\lambda\in\mathbb{R}$.}
\end{align}

\item Non-uniform $\mathcal{Y}$-MGF bounds: We have $\xi_{\mathcal{Y}_{i,j}}:\mathcal{X}_i\times [0,\infty)\to[0,\infty]$, $i=1,\dots,n$, $j=1,\dots,m_i$, such that  $\xi_{\mathcal{Y}_{i,j}}(\cdot,\lambda)$ is measurable for all  $\lambda \in[0,\infty)$ and   
\begin{align}\label{eq:gen_McDiarmid_non_unif_Y_MGF_assump}
\int \cosh(\lambda  h_{\mathcal{Y}_{i,j}}(x_i,\cdot,\cdot)) dP_{{Y}_{i,j}}^2\leq \exp\left(\xi_{\mathcal{Y}_{i,j}}(x_i,|\lambda|)\right) \,\,\,\text{  for all  $x_i\in\mathcal{X}_i$, $\lambda\in\mathbb{R}$.}
\end{align}
\end{enumerate}
Here and in the following we use the definitions $0\cdot \infty\coloneqq0$, $\cosh(\pm\infty)\coloneqq\infty$.

 For  $i=1,\dots,n$, let $p_i\in[1,\infty)$ and let $q_i$ be the conjugate exponent to $p_i$.  Then for all $\lambda\in\mathbb{R}$ we have the MGF bound
 \begin{align}
 &\mathbb{E}\left[\exp\left(\lambda(\phi(Z)-\mathbb{E}[\phi(Z)])\right)\right]
 \leq \exp\left(\sum_{i=1}^np_i^{-1}\xi_{\mathcal{X}_i}(|\lambda p_i|)\right)\prod_{i=1}^n\left\|\exp\left(\sum_{j=1}^{m_i}\xi_{\mathcal{Y}_{i,j}}(\cdot,|\lambda|)\right)\right\|_{L^{q_i}(P_{X_i})}
 \end{align}
 and for all $\epsilon>0$ we have the concentration inequality
 \begin{align}\label{eq:gen_McDiarmid_non_unif_conc_ineq}
 &\mathbb{P}\left(\pm(\phi(Z)-\mathbb{E}[\phi(Z)])\geq \epsilon\right)\\
 \leq& \inf_{\lambda\geq 0}\left\{  \exp\left(-\lambda \epsilon+\sum_{i=1}^np_i^{-1}\xi_{\mathcal{X}_i}(\lambda p_i)\right)\prod_{i=1}^n\left\|\exp\left(\sum_{j=1}^{m_i}\xi_{\mathcal{Y}_{i,j}}(\cdot,\lambda)\right)\right\|_{L^{q_i}(P_{X_i})}\right\}\,.\notag
 \end{align}
\end{theorem}
\begin{remark} 

Note that   \eqref{eq:gen_McDiarmid_non_unif_X_MGF_assump} is equivalent to an MGF bound for $\zeta h_{\mathcal{X}_i}(x_i,\tilde x_i)$ under the distribution $d\zeta\times P_{X_i}(dx_i)\times P_{X_i}(d\tilde{x}_i)$, where $d \zeta$ is the uniform distribution on $\{-1,1\}$, and similarly for \eqref{eq:gen_McDiarmid_non_unif_Y_MGF_assump}.  Thus we   refer to  \eqref{eq:gen_McDiarmid_non_unif_X_MGF_assump}-\eqref{eq:gen_McDiarmid_non_unif_Y_MGF_assump} as ($\cosh$-form) MGF bounds.  This observation was used in the earlier sub-Gaussian diameter approach of  \cite{kontorovich2014concentration}.

We emphasize that the key innovation of the present result, as compared to the related works discussed in Section \ref{sec:related_works}, is the ability to treat  (unbounded) non-uniform dependence of the mean-difference bound for $\mathcal{Y}$ on the corresponding $\mathcal{X}$-variable,  see Eq.~\eqref{eq:Y_mean_diff_bound}, with corresponding non-uniform MGF bound \eqref{eq:gen_McDiarmid_non_unif_Y_MGF_assump}.

\end{remark}
\begin{proof}
Define $\phi_n\coloneqq\phi$,    $\phi_0\coloneqq \mathbb{E}[\phi(Z)]$,
\begin{align}
&\phi_{i-1}(z^{1:i-1})\coloneqq\mathbb{E}[\phi(z^{1:i-1},Z^{i:n})]\,,\\
&\phi_{i,\ell}(z^{1:i-1},x_i,y_i^{1:\ell})\coloneqq\mathbb{E}[\phi(z^{1:i-1},x_i,y_i^{1:\ell},Y_i^{\ell+1:m_i},Z^{i+1:n})]\,,\notag
\end{align}
and note that  $\phi_{i,m_i}=\phi_{i}$ for all $i$; the integrability assumptions imply these are all real-valued.  For $i=1,\dots,n$ define $V_i\coloneqq \phi_{i}-\phi_{i-1}$ and  $V_{i,0}\coloneqq \phi_{i,0}-\phi_{i-1}$. For $i=1,\dots,n$, $j=1,\dots,m_i$, define $V_{i,j}\coloneqq\phi_{i,j}-\phi_{i,j-1}$; note that $V_{i}$ depends on $z^{1:i}$ and $V_{i,j}$ depends on $(z^{1:i-1},x_i,y_i^{1:j})$. We have the telescoping sums
\begin{align}
\phi-\mathbb{E}[\phi(Z)]=\sum_{i=1}^nV_i\,,\,\,\,\,\,
 V_i= \sum_{j=0}^{m_i}V_{i,j} \text{ for }i=1,\dots,n\,.
\end{align}

The assumed  one-component  mean-difference bounds for $\mathcal{Y}$ imply
 \begin{align}
 \left|\phi_{i,j}(z^{1:i-1},x_i,y_i^{1:j-1},y_{i,j})- 
 \phi_{i,j}(z^{1:i-1},x_i,y_i^{1:j-1},\tilde{y}_{i,j}) 
 \right|\leq  h_{\mathcal{Y}_{i,j}}(x_i,y_{i,j},\tilde{y}_{i,j})\,.
 \end{align}

For all $i=1,\dots,n$, $j=1,\dots,m_i$, and all $\lambda\in\mathbb{R}$ we can use Jensen's inequality, symmetrization, and properties of the hyperbolic cosine to compute
\begin{align}
 &\mathbb{E}\left[\exp\left(\lambda V_{i,j}(z^{1:i-1},x_i,y_i^{1:j-1},Y_{i,j})\right)\right]\\
 =&\mathbb{E}\left[\exp\left(\lambda (\phi_{i,j}(z^{1:i-1},x_i,y_i^{1:j-1},Y_{i,j})-\phi_{i,j-1}(z^{1:i-1},x_i,y_i^{1:j-1}))\right)\right]\notag\\
  =&\int \exp\left(\lambda \left(\phi_{i,j}(z^{1:i-1},x_i,y_i^{1:j-1},y_{i,j})-E_{P_{Y_{i,j}}}[ \phi_{i,j}(z^{1:i-1},x_i,y_i^{1:j-1},\cdot)]\right)\right)P_{Y_{i,j}}(dy_{i,j})\notag\\
    \leq&\int\!\!\! \int\exp\left(\lambda \left(\phi_{i,j}(z^{1:i-1},x_i,y_i^{1:j-1},y_{i,j})- \phi_{i,j}(z^{1:i-1},x_i,y_i^{1:j-1},\tilde{y}_{i,j}) \right)\right) P_{Y_{i,j}}(d\tilde{y}_{i,j})P_{Y_{i,j}}(dy_{i,j})\notag\\
        =&\int\!\!\! \int\cosh\left(\lambda \left(\phi_{i,j}(z^{1:i-1},x_i,y_i^{1:j-1},y_{i,j})- \phi_{i,j}(z^{1:i-1},x_i,y_i^{1:j-1},\tilde{y}_{i,j}) \right)\right) P_{Y_{i,j}}(d\tilde{y}_{i,j})P_{Y_{i,j}}(dy_{i,j})\notag\\
        \leq&\int\!\!\! \int\cosh\left(|\lambda|  h_{\mathcal{Y}_{i,j}}(x_i,y_{i,j},\tilde{y}_{i,j})\right) P_{Y_{i,j}}(d\tilde{y}_{i,j})P_{Y_{i,j}}(dy_{i,j})\notag\\
        \leq&   \exp\left(\xi_{\mathcal{Y}_{i,j}}(x_i,|\lambda|)\right)\,.\notag
\end{align}
Next, by employing the above along with  induction on $\ell$ we obtain
\begin{align}
\mathbb{E}\left[\prod_{j=1}^\ell \exp\left(\lambda V_{i,j}(z^{1:i-1},x_i,Y_i^{1:j})\right)\right]\leq \exp\left(\sum_{j=1}^\ell \xi_{\mathcal{Y}_{i,j}}(x_i,|\lambda|)\right)\,.
\end{align}
Hence  for $i=1,\dots,n$ we can compute
\begin{align}\label{eq:Vi_MGF_bound}
&\mathbb{E}\left[\exp\left(\lambda V_i(z^{1:i-1},Z_i)\right)\right]\\
=&\mathbb{E}\left[\exp\left(\lambda  \sum_{j=0}^{m_i}V_{i,j}(z^{1:i-1},X_i,Y_i^{1:j})\right)\right]\notag\\
=&\int \exp\left(\lambda  V_{i,0}(z^{1:i-1},x_i )\right)\mathbb{E}\left[\prod_{j=1}^{m_i}\exp\left(\lambda  V_{i,j}(z^{1:i-1},x_i,Y_i^{1:j})\right)\right]P_{X_i}(dx_i)\notag\\
\leq&\int \exp\left(\lambda  V_{i,0}(z^{1:i-1},x_i )\right)\exp\left(\sum_{j=1}^{m_i} \xi_{\mathcal{Y}_{i,j}}(x_i,|\lambda|)\right)P_{X_i}(dx_i)\notag\\
\leq& \left(\int \exp\left(\lambda p_i  V_{i,0}(z^{1:i-1},x_i )\right) P_{X_i}(dx_i)\right)^{1/p_i}
\left\|\exp\left(\sum_{j=1}^{m_i} \xi_{\mathcal{Y}_{i,j}}(\cdot,|\lambda|)\right)\right\|_{L^{q_i}(P_{X_i})}\,, \notag
\end{align}
where we used H{\"o}lder's inequality to obtain the last line.

By a similar computation, for all $i=1,\dots,n$, $\kappa\in\mathbb{R}$,  we obtain
\begin{align}\label{eq:Vi0_MGF_bound}
&\int\exp\left(\kappa V_{i,0}(z^{1:i-1},x_i)\right)P_{X_i}(dx_i)
\leq \exp(\xi_{\mathcal{X}_i}(|\kappa|))\,.
\end{align}
 
 Combining \eqref{eq:Vi_MGF_bound} with \eqref{eq:Vi0_MGF_bound}  we find
\begin{align}\label{eq:Vi_MGF_bound_final}
\mathbb{E}\left[\exp\left(\lambda V_i(z^{1:i-1},Z_i)\right)\right]
\leq& \exp(p_i^{-1}\xi_{\mathcal{X}_i}(|\lambda p_i |))
\left\|\exp\left(\sum_{j=1}^{m_i} \xi_{\mathcal{Y}_{i,j}}(\cdot,|\lambda|)\right)\right\|_{L^{q_i}(P_{X_i})}\,.
\end{align}

Finally, using \eqref{eq:Vi_MGF_bound_final} along with induction we obtain
\begin{align}
\mathbb{E}\left[\exp\left(\lambda(\phi(Z)-\mathbb{E}[\phi(Z)])\right)\right]
=&\mathbb{E}\left[\prod_{i=1}^n\exp\left(\lambda V_i(Z^{1:i})\right)\right]\\
\leq&\prod_{i=1}^n \exp\left(p_i^{-1}\xi_{\mathcal{X}_i}(|\lambda p_i|)\right)\left\|\exp\left(\sum_{j=1}^{m_i}\xi_{\mathcal{Y}_{i,j}}(\cdot,|\lambda|)\right)\right\|_{L^{q_i}(P_{X_i})}\,.\notag
\end{align}
This proves the claimed MGF bound. 
The claimed concentration inequality then follows via a standard Chernoff bound.
\end{proof}
The key to  using Theorem \ref{thm:gen_McDiarmid_non_unif_MGF} in practice is obtaining  MGF bounds of the form \eqref{eq:gen_McDiarmid_non_unif_Y_MGF_assump}  that will translate to a non-trivial upper bound in \eqref{eq:gen_McDiarmid_non_unif_conc_ineq}. More precisely, the difficulty lies in the need to obtain $\xi_{\mathcal{Y}_{i,j}}(x_i,|\lambda|)$ that are finite in a non-trivial  $\lambda$-interval for a.e. $x_i$; when $x_i$ is unbounded, as in our application to DSM, this is generally a non-trivial task.  In the following subsection we will develop a technique for accomplishing this by splitting the bounding functions $h_{\mathcal{Y}_{i,j}}$ into multiple terms which are controlled by different Orlicz norms and appropriately combining those bounds.

\subsection{Cosh-form MGF Bounds via Orlicz Norms}\label{sec:Orlicz}
Given a  (non-decreasing) Young function $\Psi$, we will denote the corresponding
Orlicz (Luxemburg) norm of a random variable $X$ by $\|X\|_\Psi$ (or $\|X\|_{\Psi,{P}}$ if we want to emphasize the role of the underlying probability measure, $P$). For the definition of and background on   Orlicz norms see, e.g., Chapter 2.2 in \cite{van2013weak}.   We will make use of the following bound on the $\cosh$-form MGF of a sum of random variables, where each term in the sum is able to use a different Orlicz norm; this generalizes Lemma 2.3 in \cite{doi:10.1137/25M1748664}, which covered the $k=1$ case. The ability to use a different Orlicz norm for each term will be key for obtaining non-trivial  bounds of the form \eqref{eq:gen_McDiarmid_non_unif_Y_MGF_assump} in our application to DSM.
\begin{lemma}\label{lemma:Orlicz_cosh_MGF_bound_sum}
Let $\Psi_i$, $i=1,\dots,k$ be  non-decreasing Young functions, $(\Omega,\mathcal{F},P)$ be a probability space and $X_i$, $i=1,\dots,k$ be  real-valued random variables on $\Omega$ with $\|X_i\|_{\Psi_i}<\infty$ for all $i$. Then  for all   $s_i\in(0,1]$, $i=1,\dots,k$, satisfying $\sum_i s_i=1$ and all $\lambda\in\mathbb{R}$ we have
\begin{align}\label{eq:cosh_Orlicz_bound}
E_P\left[\cosh\left(\lambda \sum_{i=1}^k X_i\right)\right]\leq&1+\sum_{i=1}^k\sum_{n=1}^\infty  \frac{\lambda^{2n}}{(2n)!} \int_0^\infty\min\left\{\frac{1}{\Psi_i(s_i t^{1/(2n)}/\|X_i\|_{\Psi_i})},1\right\}dt \,,
\end{align}
where we define 
\end{lemma}
\begin{remark}
Here and in the following, we will use the conventions $1/\infty\coloneqq0$, $a/0\coloneqq \infty$ for all $a>0$, and $\Psi(\infty)\coloneqq \infty$ for all non-decreasing Young functions $\Psi$.
\end{remark}
\begin{proof}
If any of the $\|X_i\|_{\Psi_i}=0$ then our conventions ensure they drop  out of both sides of the bound \eqref{eq:cosh_Orlicz_bound}, so we can suppose that $\|X_i\|_{\Psi_i}>0$ for all $i$.

Recall   the  concentration inequality 
\begin{align}\label{eq:general_Orlicz_conc_ineq}
{P}(|X_i|> s)\leq \min\left\{\frac{1}{\Psi_i(s/\|X_i\|_{\Psi_i})},1\right\}
\end{align}
for all $s>0$; see, e.g., page 96 in  \cite{van2013weak}. Using a union bound and then applying \eqref{eq:general_Orlicz_conc_ineq} with  $s=s_i t$, $i=1,\dots,k$, we find
\begin{align}
P\left(\left|\sum_{i=1}^k X_i\right|> t\right)\leq &\sum_{i=1}^k P(|X_i|> t s_i)\\
\leq& \sum_{i=1}^k\min\left\{\frac{1}{\Psi_i(s_i t/\|X_i\|_{\Psi_i})},1\right\}\notag\,.
\end{align}
Therefore for $p>0$ we obtain
\begin{align}
E_P\left[\left|\sum_{i=1}^kX_i\right|^p\right]=&\int_0^\infty P\left(\left|\sum_{i=1}^kX_i\right|> t^{1/p}\right)dt \\
\leq& \sum_{i=1}^k \int_0^\infty\min\left\{\frac{1}{\Psi_i(s_i t^{1/p}/\|X_i\|_{\Psi_i})},1\right\}dt \,.\notag
\end{align}
Finally, using the Taylor series for the hyperbolic cosine we can compute
\begin{align}
E_P\left[\cosh\left(\lambda \sum_{i=1}^k X_i\right)\right]=&\sum_{n=0}^\infty \frac{\lambda^{2n}}{(2n)!}E_P\left[\left( \sum_{i=1}^k X_i\right)^{2n}\right]\\
\leq&1+\sum_{i=1}^k\sum_{n=1}^\infty  \frac{\lambda^{2n}}{(2n)!} \int_0^\infty\min\left\{\frac{1}{\Psi_i(s_i t^{1/(2n)}/\|X_i\|_{\Psi_i})},1\right\}dt \,.\notag
\end{align}
\end{proof}

 Our focus  will be on the increasing Young functions 
 \begin{align}\label{eq:Phi_q_def}
 \Phi_q(t)\coloneqq \exp(t^q)-1\,,\,\,\,\, q\in[1,\infty)\,.
 \end{align}
  Recall that  $\|X\|_{\Phi_1}<\infty$ is equivalent to $X$ being sub-exponential and $\|X\|_{\Phi_2}<\infty$ is equivalent to  $X$ being sub-Gaussian. We will also have need for the following bound, which uses the Orlicz norm induced by \eqref{eq:Phi_q_def}.
\begin{lemma}\label{lemma:Orlicz_exp_q_moment_bounds}
Let $(\Omega,\mathcal{F},P)$ be a probability space,   $q\in[1,\infty)$, and $X$ be a real-valued random variable on $\Omega$ with $\|X\|_{\Phi_q}<\infty$.

For $\alpha\geq 0$ satisfying $\alpha \|X\|_{\Phi_q}^q<1$ we have
\begin{align}
E_P[e^{\alpha |X|^q}]\leq & 1+2\frac{\alpha\|X\|_{\Phi_q}^q}{1-\alpha\|X\|_{\Phi_q}^q}\,.
\end{align} 
\end{lemma}
\begin{proof}
The cases where $\alpha=0$ or $\|X\|_{\Phi_q}=0$ are trivial, so suppose they are both positive.
 Letting  $h(x)=e^{\alpha x^q}-1$ and again using \eqref{eq:general_Orlicz_conc_ineq} we can compute 
 \begin{align}
 E_P[h(|X|)]=&\int_0^\infty P\left(h(|X|)\geq t\right)dt \\
 =&\int_0^\infty P\left(|X|\geq h^{-1}(t)\right)dt\notag\\
\leq &\int_0^\infty \min\left\{\frac{1}{\Phi_q((\alpha^{-1}\log(t+1))^{1/q}/\|X\|_{\Phi_q})},1\right\}dt\,.\notag
 \end{align}
Simplifying this, under the assumption that $\alpha \|X\|_{\Phi_q}^q<1$, we obtain
\begin{align}
E_P[e^{\alpha |X|^q}]\leq & 1+\int_0^\infty \min\left\{\frac{1}{\Phi_q((\alpha^{-1}\log(t+1))^{1/q}/\|X\|_{\Phi_q})},1\right\}dt\\
\leq& 1+2\int_0^\infty \exp(-\alpha^{-1}\log(t+1)/\|X\|^q_{\Phi_q})dt\notag\\
=& 1+2\frac{\alpha\|X\|_{\Phi_q}^q}{1-\alpha\|X\|_{\Phi_q}^q}\notag\,.
\end{align}
\end{proof}

 \subsection{Concentration Inequalities in Terms of Orlicz-Norm Bounds}
 Now we will combine Theorem \ref{thm:gen_McDiarmid_non_unif_MGF} with Lemmas \ref{lemma:Orlicz_cosh_MGF_bound_sum} and \ref{lemma:Orlicz_exp_q_moment_bounds} to obtain a   concentration inequality in terms of Orlicz norm bounds on the $\mathcal{Y}$-component  mean-difference bounding functions.  Our focus is on cases where those bounding functions can be decomposed into appropriate sub-exponential and sub-Gaussian terms; as we will  show in Section \ref{sec:DSM}, this matches the structure of the DSM problem.  
\begin{theorem}\label{thm:ULLN_bound_non_unif_MGF_Phi1_Phi2_case}
Under Assumption \ref{asump:gen_McDiarmid_non_unif_MGF}, suppose we also have the following:
\begin{enumerate}
\item MGF bounds on $\mathcal{X}$: For $i=1,\dots,n$, we have $\xi_{\mathcal{X}_i}:[0,\infty)\to[0,\infty]$ such that 
\begin{align}\label{eq:h_Xi_cosh_MGF_bound}
\int \cosh(\lambda  h_{\mathcal{X}_i}) dP_{{X}_i}^2\leq \exp(\xi_{\mathcal{X}_i}(|\lambda|)) \,\,\,\text{  for all $\lambda\in\mathbb{R}$.}
\end{align}

\item  For  $i=1,\dots,n$, $j=1,\dots,m_i$, we have 
 measurable ${h}_{0,\mathcal{Y}_{i,j}}:\mathcal{Y}_{i,j}\times \mathcal{Y}_{i,j}\to[0,\infty)$,  $r_{\mathcal{X}_i}:\mathcal{X}_i\to[0,\infty)$, $k_{\mathcal{Y}_{i,j}}:\mathcal{Y}_{i,j}\times\mathcal{Y}_{i,j}\to[0,\infty)$
 such that
\begin{align}
h_{\mathcal{Y}_{i,j}}(x_i,y_{i,j},\tilde{y}_{i,j})= {h}_{0,\mathcal{Y}_{i,j}}( y_{i,j},\tilde{y}_{i,j})+r_{\mathcal{X}_i}(x_i)k_{\mathcal{Y}_{i,j}}(y_{i,j},\tilde{y}_{i,j})
\end{align}
and
\begin{align}
\|{h}_{0,\mathcal{Y}_{i,j}}\|_{\Phi_{1},P^2_{Y_{i,j}}}<\infty\,,\,\,\,\,\,\,
\|k_{\mathcal{Y}_{i,j}}\|_{\Phi_2,P^2_{Y_{i,j}}}<\infty\,, \,\,\,\,\|r_{\mathcal{X}_i}\|_{\Phi_2,P_{X_i}}<\infty\,,
\end{align}
  where the $\Phi_q$ were defined in \eqref{eq:Phi_q_def}.
\end{enumerate}
For  $i=1,\dots,n$, let    $p_i\in(1,\infty)$ and $q_i$ be the conjugate exponent to $p_i$, and for  $i=1,\dots,n$, $j=1,\dots,m_i$,  let $s_{i,j}\in(0,1)$.  Then   for all $\epsilon> 0$ we have the concentration inequality
        \begin{align} \label{eq:ULLN_bound_non_unif_MGF_Phi1_Phi2_case}
   &\mathbb{P}\left(\pm(\phi(Z)-\mathbb{E}[\phi(Z)])\geq \epsilon\right)
   \leq  \exp\left(-\sup_{\lambda\in[ 0,\lambda_{\max})}\left\{\lambda \epsilon-\sum_{i=1}^n\xi_{\mathcal{X}_i,m_i}(\lambda)\right\}\right)\,,
   \end{align}
   where
   \begin{align}
   \xi_{\mathcal{X}_i,m_i}(\lambda)\coloneqq&\, p_i^{-1}\xi_{\mathcal{X}_i}(\lambda p_i)+\frac{ 2 \lambda^2\sum_{j=1}^{m_i}\alpha_{i,j}\|r_{\mathcal{X}_i}\|_{\Phi_2,P_{X_i}}^2}{1-q_i\lambda^2\sum_{j=1}^{m_i}\alpha_{i,j}\|r_{\mathcal{X}_i}\|_{\Phi_2,P_{X_i}}^2}+
      \sum_{j=1}^{m_i}\frac{2\left(\lambda\beta_{i,j}\right)^{2}}{1-\left(\lambda\beta_{i,j}\right)^{2}}\,,\\   
      \lambda_{\max}\coloneqq &\min\left\{\min_{i,j} 1/\beta_{i,j},\min_i\left(q_i\|r_{\mathcal{X}_i}\|_{\Phi_2,P_{X_i}}^2\sum_{j=1}^{m_i}\alpha_{i,j}\right)^{-1/2}\right\}\,,\label{eq:lambda_max_def}\\
   \alpha_{i,j}\coloneqq & {\|k_{\mathcal{Y}_{i,j}}\|^2_{\Phi_2,P^2_{Y_{i,j}}}}/{s_{i,j}^2}\,,\,\,\,\,\,\,
           \beta_{i,j}\coloneqq \|{h}_{0,\mathcal{Y}_{i,j}}\|_{\Phi_{1},P^2_{Y_{i,j}}}/({1-s_{i,j}})\,.\label{eq:alpha_ij_beta_ij_def}
        \end{align}   
\end{theorem}
\begin{remark}
For now we leave the MGF bound \eqref{eq:h_Xi_cosh_MGF_bound} in the general form as it is much less delicate to handle than the $Y$-MGF terms;   bounds of the form \eqref{eq:h_Xi_cosh_MGF_bound} can be obtained via standard techniques, see \cite{boucheron2013concentration,vershynin2018high,wainwright2019high}, and also via the $k=1$ case of Lemma \ref{lemma:Orlicz_cosh_MGF_bound_sum}, as found in  \cite{doi:10.1137/25M1748664}, for any choice of non-decreasing Young function.
\end{remark}
\begin{proof}
Define $\Delta h_{\mathcal{Y}_{i,j}}(x_i,y_{i,j},\tilde{y}_{i,j})\coloneqq r_{\mathcal{X}_i}(x_i)k_{\mathcal{Y}_{i,j}}(y_{i,j},\tilde{y}_{i,j})$. 
First we verify the condition \eqref{eq:gen_McDiarmid_non_unif_Y_MGF_assump}.  For all $x_i\in\mathcal{X}_i$, $\lambda\in\mathbb{R}$, we can use Lemma \ref{lemma:Orlicz_cosh_MGF_bound_sum} to compute
\begin{align}
&\int \cosh(\lambda  h_{\mathcal{Y}_{i,j}}(x_i,\cdot,\cdot)) dP_{{Y}_{i,j}}^2\\
=&\int \cosh(\lambda ({h}_{0,\mathcal{Y}_{i,j}}+\Delta{h}_{\mathcal{Y}_{i,j}}(x_i,\cdot,\cdot))) dP_{{Y}_{i,j}}^2\notag\\
\leq&1+\sum_{\ell=1}^\infty  \frac{\lambda^{2\ell}}{(2\ell)!} \int_0^\infty\min\left\{\frac{1}{\Phi_2(s_{i,j} t^{1/(2\ell)}/\|\Delta {h}_{\mathcal{Y}_{i,j}}(x_i,\cdot,\cdot)\|_{\Phi_2,P^2_{Y_{i,j}}})},1\right\}dt \notag\\
&+\sum_{\ell=1}^\infty  \frac{\lambda^{2\ell}}{(2\ell)!} \int_0^\infty\min\left\{\frac{1}{\Phi_{1}((1-s_{i,j}) t^{1/(2\ell)}/\|{h}_{0,\mathcal{Y}_{i,j}}\|_{\Phi_1,P^2_{Y_{i,j}}})},1\right\}dt \,.\notag
\end{align}
Using the bounds $\min\{a^{-1},1\}\leq 2/(a+1)$ for all $a\in[0,\infty]$ and $(2\ell)!\geq 2(\ell!)^2$ for all $\ell\in\mathbb{Z}^+$, we bound the first series as follows:
 \begin{align}
 &1+\sum_{\ell=1}^\infty  \frac{\lambda^{2\ell}}{(2\ell)!} \int_0^\infty\min\left\{\frac{1}{\Phi_2(s_{i,j} t^{1/(2\ell)}/\|\Delta {h}_{\mathcal{Y}_{i,j}}(x_i,\cdot,\cdot)\|_{\Phi_2,P^2_{Y_{i,j}}})},1\right\}dt \\
 \leq&1+\sum_{\ell=1}^\infty  \frac{\lambda^{2\ell}}{(2\ell)!} \int_0^\infty \frac{2}{1+\Phi_2(s_{i,j} t^{1/(2\ell)}/\|\Delta {h}_{\mathcal{Y}_{i,j}}(x_i,\cdot,\cdot)\|_{\Phi_2,P^2_{Y_{i,j}}})} dt\notag \\
  =&1+\sum_{\ell=1}^\infty  \frac{2\lambda^{2\ell}}{(2\ell)!} \int_0^\infty  \exp\left(-s_{i,j}^2 t^{1/\ell}/\|\Delta {h}_{\mathcal{Y}_{i,j}}(x_i,\cdot,\cdot)\|^2_{\Phi_2,P^2_{Y_{i,j}}}\right)  dt \notag\\
  =&1+\sum_{\ell=1}^\infty  \frac{2\lambda^{2\ell}}{(2\ell)!}\ell!\left(\|\Delta {h}_{\mathcal{Y}_{i,j}}(x_i,\cdot,\cdot)\|^2_{\Phi_2,P^2_{Y_{i,j}}}/s_{i,j}^2\right)^\ell\notag\\
  \leq&1+\sum_{\ell=1}^\infty  \frac{1}{\ell!} \left(\lambda^{2}\|\Delta {h}_{\mathcal{Y}_{i,j}}(x_i,\cdot,\cdot)\|^2_{\Phi_2,P^2_{Y_{i,j}}}/s_{i,j}^2\right)^\ell=\exp\left(\lambda^{2}\|\Delta {h}_{\mathcal{Y}_{i,j}}(x_i,\cdot,\cdot)\|^2_{\Phi_2,P^2_{Y_{i,j}}}/s_{i,j}^2\right)\,.\notag
 \end{align}
Similarly,  for $|\lambda |<({1-s_{i,j}})/\|{h}_{0,\mathcal{Y}_{i,j}}\|_{\Phi_{1},P^2_{Y_{i,j}}}$ we can bound the remaining series as follows:
\begin{align}
&\sum_{n=1}^\infty  \frac{\lambda^{2n}}{(2n)!} \int_0^\infty\min\left\{\frac{1}{\Phi_{1}((1-s_{i,j}) t^{1/(2n)}/\|{h}_{0,\mathcal{Y}_{i,j}}\|_{\Phi_{1},P^2_{Y_{i,j}}})},1\right\}dt\\
\leq&\sum_{n=1}^\infty  \frac{2\lambda^{2n}}{(2n)!} \int_0^\infty\exp(-(1-s_{i,j}) t^{1/(2n)}/\|{h}_{0,\mathcal{Y}_{i,j}}\|_{\Phi_{1},P^2_{Y_{i,j}}})dt\notag\\
=&2\sum_{n=1}^\infty   \left(\frac{\lambda\|{h}_{0,\mathcal{Y}_{i,j}}\|_{\Phi_{1},P^2_{Y_{i,j}}}}{1-s_{i,j}}\right)^{2n}
=\frac{2\left({\lambda\|{h}_{0,\mathcal{Y}_{i,j}}\|_{\Phi_{1},P^2_{Y_{i,j}}}}/({1-s_{i,j}})\right)^{2}}{1-\left({\lambda\|{h}_{0,\mathcal{Y}_{i,j}}\|_{\Phi_{1},P^2_{Y_{i,j}}}}/({1-s_{i,j}})\right)^{2}}\,.\notag
\end{align}
  Combining these, for all    $|\lambda|<({1-s_{i,j}})/\|{h}_{0,\mathcal{Y}_{i,j}}\|_{\Phi_{1},P^2_{Y_{i,j}}}$ we obtain
\begin{align}
&\int \cosh(\lambda  h_{\mathcal{Y}_{i,j}}(x_i,\cdot,\cdot)) dP_{{Y}_{i,j}}^2\\
\leq&\exp\left(\lambda^{2}\|\Delta {h}_{\mathcal{Y}_{i,j}}(x_i,\cdot,\cdot)\|^2_{\Phi_2,P^2_{Y_{i,j}}}/s_{i,j}^2\right)+\frac{2\left({\lambda\|{h}_{0,\mathcal{Y}_{i,j}}\|_{\Phi_{1},P^2_{Y_{i,j}}}}/({1-s_{i,j}})\right)^{2}}{1-\left({\lambda\|{h}_{0,\mathcal{Y}_{i,j}}\|_{\Phi_{1},P^2_{Y_{i,j}}}}/({1-s_{i,j}})\right)^{2}}\notag\\
\leq&\exp\left(\lambda^{2}(r_{\mathcal{X}_i}(x_i))^2\|k_{\mathcal{Y}_{i,j}}\|^2_{\Phi_2,P^2_{Y_{i,j}}}/s_{i,j}^2+\frac{2\left({\lambda\|{h}_{0,\mathcal{Y}_{i,j}}\|_{\Phi_{1},P^2_{Y_{i,j}}}}/({1-s_{i,j}})\right)^{2}}{1-\left({\lambda\|{h}_{0,\mathcal{Y}_{i,j}}\|_{\Phi_{1},P^2_{Y_{i,j}}}}/({1-s_{i,j}})\right)^{2}}\right)\,.\notag
\end{align}

We have confirmed that all of the assumptions of Theorem \ref{thm:gen_McDiarmid_non_unif_MGF} hold, hence we can conclude
  \begin{align} \label{eq:conv_ineq_sub_Gauss_temp}
  &\mathbb{P}\left(\pm(\phi(Z)-\mathbb{E}[\phi(Z)])\geq \epsilon\right)\\
  \leq& \inf_{\lambda\geq 0}\left\{  \exp\left(-\lambda \epsilon+\sum_{i=1}^np_i^{-1}\xi_{\mathcal{X}_i}(\lambda p_i)\right)\prod_{i=1}^n\left\|\exp\left(\sum_{j=1}^{m_i}\xi_{\mathcal{Y}_{i,j}}(\cdot,\lambda)\right)\right\|_{L^{q_i}(P_{X_i})}\right\}\,,\notag
  \end{align}
  where  $\alpha_{i,j}$, $\beta_{i,j}$ were defined in \eqref{eq:alpha_ij_beta_ij_def} and
 \begin{align}
  \xi_{\mathcal{Y}_{i,j}}(x_i,\lambda)\coloneqq &\begin{cases}
       \alpha_{i,j}(r_{\mathcal{X}_i}(x_i))^2 \lambda^{2}
        +\frac{2\left(\lambda\beta_{i,j}\right)^{2}}{1-\left(\lambda \beta_{i,j}\right)^{2}} &\text{ if }\lambda\in[0,1/\beta_{i,j})\\
        \infty &\text{ otherwise}
        \end{cases}\,.
   \end{align}
 
  With $\lambda_{\max}$ as defined in \eqref{eq:lambda_max_def},  for all $\lambda\in[0,\lambda_{\max})$   we can compute
\begin{align}\label{eq:xi_Y_MGF_bound_sub_Gauss}
&\left\|\exp\left(\sum_{j=1}^{m_i}\xi_{\mathcal{Y}_{i,j}}(\cdot,\lambda)\right)\right\|_{L^{q_i}(P_{X_i})} 
\\
\leq& \mathbb{E}\left[ \exp\left(q_i\lambda^{2}(r_{\mathcal{X}_i}(X_i))^2\sum_{j=1}^{m_i}\alpha_{i,j}\right) \right]^{1/q_i} \exp\left(
 \sum_{j=1}^{m_i}\frac{2\left(\lambda\beta_{i,j}\right)^{2}}{1-\left(\lambda\beta_{i,j}\right)^{2}}\right)\notag\\
\leq& \left(1+2\frac{q_i\lambda^2\sum_{j=1}^{m_i}\alpha_{i,j}\|r_{\mathcal{X}_i}\|_{\Phi_2,P_{X_i}}^2}{1-q_i\lambda^2\sum_{j=1}^{m_i}\alpha_{i,j}\|r_{\mathcal{X}_i}\|_{\Phi_2,P_{X_i}}^2}\right)^{1/q_i}
 \exp\left(
 \sum_{j=1}^{m_i}\frac{2\left(\lambda\beta_{i,j}\right)^{2}}{1-\left(\lambda\beta_{i,j}\right)^{2}}\right)\notag\\
\leq& 
 \exp\left(\frac{ 2 \lambda^2\sum_{j=1}^{m_i}\alpha_{i,j}\|r_{\mathcal{X}_i}\|_{\Phi_2,P_{X_i}}^2}{1-q_i\lambda^2\sum_{j=1}^{m_i}\alpha_{i,j}\|r_{\mathcal{X}_i}\|_{\Phi_2,P_{X_i}}^2}+
 \sum_{j=1}^{m_i}\frac{2\left(\lambda\beta_{i,j}\right)^{2}}{1-\left(\lambda\beta_{i,j}\right)^{2}}\right)\,,\notag
\end{align}
   where we used   Lemma \ref{lemma:Orlicz_exp_q_moment_bounds} to bound the MGF of $r_{\mathcal{X}_i}^2$.
   Combining \eqref{eq:conv_ineq_sub_Gauss_temp}-\eqref{eq:xi_Y_MGF_bound_sub_Gauss} and simplifying, we arrive at the claimed concentration inequality.      
\end{proof}

\section{ULLN and Stochastic Optimization}\label{sec:ULLN_stoch_opt}
Theorems \ref{thm:gen_McDiarmid_non_unif_MGF} and \ref{thm:ULLN_bound_non_unif_MGF_Phi1_Phi2_case} are the new general-purpose concentration inequalities   developed in this paper.  In this section we will use them to derive a corresponding ULLN result as well as concentration inequalities for stochastic optimization problems; see Theorems \ref{thm:ULLN_non_uniform_mean_diff_bounds} and \ref{thm:stoch_opt} below.  
\subsection{A ULLN Concentration Inequality with Non-Uniform MGF Bounds}
In this section we use Theorem \ref{thm:ULLN_bound_non_unif_MGF_Phi1_Phi2_case} to derive a novel ULLN result that allows for sample reuse and non-uniform MGF bounds. We will work under the following assumptions.
\begin{assumption}\label{assump:ULLN}
Let $(\mathcal{X},P_{ {X}})$, $(\mathcal{Y},P_{ {Y}})$ be probability spaces and suppose   the following:
\begin{enumerate}
\item We have a nonempty collection $\mathcal{G}$ of measurable functions $g:\mathcal{X}\times\mathcal{Y}\to\mathbb{R}$.
\item We have measurable $h_{\mathcal{X}}:\mathcal{X}\times\mathcal{X}\times\mathcal{Y}\to[0,\infty)$ such that 
\begin{align}
|g(x,y)-g(\tilde{x},y)|\leq h_{\mathcal{X}}(x,\tilde{x},y) \,\,\,\,\text{ for all $x,\tilde{x}\in\mathcal{X}$, $y\in\mathcal{Y}$, $g\in\mathcal{G}$,}
\end{align}
 and $h_{\mathcal{X}}\in L^1(P_X\times P_X\times P_Y)$.
\item We have measurable $h_{\mathcal{Y}}:\mathcal{X}\times\mathcal{Y}\times\mathcal{Y}\to[0,\infty)$ such that 
\begin{align}
|g(x,y)-g(x,\tilde{y})|\leq h_{\mathcal{Y}}(x,y,\tilde{y}) \,\,\,\text{ for all $x\in\mathcal{X}$, $y,\tilde{y}\in\mathcal{Y}$, $g\in\mathcal{G}$}\,.
\end{align}
\item There exists $K_{\mathcal{X}}\in(0,\infty]$ and   $\xi_{\mathcal{X}}:[0,K_{\mathcal{X}})\to [0,\infty)$ such that
\begin{align}\label{eq:xi_X}
\int \cosh(\lambda E_{P_Y}[h_{\mathcal{X}}(x,\tilde{x},\cdot)]) P_X^2(dxd\tilde{x})\leq \exp(\xi_{\mathcal{X}}(|\lambda|))\,\,\,\,\text{ for all }\lambda\in(-K_{\mathcal{X}},K_{\mathcal{X}})\,.
\end{align}
\item There exist measurable $h_{0,\mathcal{Y}},k_{\mathcal{Y}}:\mathcal{Y}\times\mathcal{Y}\to[0,\infty)$, $r_{\mathcal{X}}:\mathcal{X}\to[0,\infty)$ such that
\begin{align}
h_{\mathcal{Y}}(x,y,\tilde{y})=h_{0,\mathcal{Y}}(y,\tilde{y})+r_{\mathcal{X}}(x)k_{\mathcal{Y}}(y,\tilde{y})
\end{align}
for all $x\in\mathcal{X}$, $y,\tilde{y}\in\mathcal{Y}$ and
\begin{align}
\|h_{0,\mathcal{Y}}\|_{\Phi_1,P_Y^2}<\infty\,,\,\,\,\|k_{\mathcal{Y}}\|_{\Phi_2,P_Y^2}<\infty\,,\,\,\,\|r_{\mathcal{X}}\|_{\Phi_2,P_X}<\infty\,.
\end{align}
\end{enumerate}
\end{assumption}
\begin{assumption}\label{assump:ULLN2}
Let $\mathcal{G}$ be a countable collection of measurable functions $g:\mathcal{X}\times\mathcal{Y}\to\mathbb{R}$ for which Assumption \ref{assump:ULLN} holds.  Suppose we have a probability space $(\Omega,\mathbb{P})$, i.i.d. $\mathcal{X}$-valued random variables $X_i$, $i=1,\dots,n$ distributed as $P_X$, and i.i.d. $\mathcal{Y}$-valued random variables $Y_{i,j}$, $i=1,\dots,n$, $j=1,\dots,m$ distributed as $P_Y$; suppose the $X_i$, $Y_{i,j}$ are all mutually independent and define
\begin{align}\label{eq:Z_def}
Z_i=(X_i,Y_{i,1},\dots,Y_{i,m})\,,\,\,\, \text{ and } Z=(Z_1,\dots,Z_n)\,.
\end{align}
\end{assumption}
Under these assumptions, we prove the following concentration inequality.  
\begin{theorem}\label{thm:ULLN_non_uniform_mean_diff_bounds}
Working under Assumption \ref{assump:ULLN2}, define $\phi_\pm:\prod_{i=1}^n(\mathcal{X}\times\mathcal{Y}^m)\to\mathbb{R}$ by
\begin{align}\label{eq:phi_pm_def}
\phi_\pm(z)\coloneqq\sup_{g\in\mathcal{G}}\left\{\pm\left(\frac{1}{nm}\sum_{i=1}^n\sum_{j=1}^m g(x_i,y_{i,j})-E_{P_X\times P_Y}[g]\right)\right\}
\end{align}
and for $t\in[0,1]$ define $\phi_t\coloneqq t\phi_++(1-t)\phi_-$.

Then   $\phi_t(Z)\in  L^1(\mathbb{P})$ and, letting  $s\in(0,1)$ and $\rho>0$   such that $\rho m>1$, for all $\epsilon>0$  we have
 \begin{align}\label{eq:ULLN_non_uniform_mean_diff_bounds}
      &\mathbb{P}\left(\pm(\phi_t(Z)-\mathbb{E}[\phi_t(Z)])\geq \epsilon\right)
      \leq  \exp\left(-n\!\sup_{\lambda\in[ 0,\lambda_{m})}\!\left\{\lambda \epsilon-\xi_{\mathcal{X},m}(\lambda)\right\}\right)\,,
    \end{align}
    where
    \begin{align}   
    \xi_{\mathcal{X},m}(\lambda)\coloneqq & (1-(\rho m)^{-1}) 
           \xi_{\mathcal{X}}\!\left((1-(\rho m)^{-1})^{-1}\lambda\right)  
          + m^{-1}\left(\frac{ 2(\alpha \lambda)^2}{1-\rho (\alpha\lambda)^2}+\frac{2(\beta\lambda )^{2}}{1- ( m^{-1}\beta\lambda)^{2}}\right)\,,\label{eq:xi_Xm_def}\\
       \lambda_{m} \coloneqq & \min\left\{ m\beta^{-1}, \left(\rho \alpha^2\right)^{-1/2} ,(1-(\rho m)^{-1}) K_{\mathcal{X}}\right\} \,,\label{eq:lambda_m_def}\\
    \alpha\coloneqq &s^{-1}\|r_{\mathcal{X}}\|_{\Phi_2,P_{X}} \|k_{\mathcal{Y}}\|_{\Phi_2,P^2_{Y}}\,,\,\,\,\,
    \beta\coloneqq  ({1-s})^{-1}\|h_{0,\mathcal{Y}}\|_{\Phi_{1},P^2_{Y}}\,.\label{eq:alpha_beta_def}
           \end{align}
 
\end{theorem}
\begin{remark}\label{remark:ULLNlarge_m_limit} 
Under appropriate assumptions on $\xi_{\mathcal{X}}$, for  sufficiently small $\rho$ and sufficiently  large $m$ we have the approximate bound
  \begin{align}\label{eq:ULLN_non_uniform_mean_diff_bounds_m_large}
       &\mathbb{P}\left(\pm(\phi_t(Z)-\mathbb{E}[\phi_t(Z)])\geq \epsilon\right)
       \lesssim  \exp\left(-n\sup_{\lambda \geq 0}\!\left\{\lambda \epsilon-  
        \xi_{\mathcal{X}}(\lambda)  
  \right\}\right)\,.
     \end{align}
     In light of \eqref{eq:xi_X}, the large $m$ (approximate) bound \eqref{eq:ULLN_non_uniform_mean_diff_bounds_m_large}
       can be viewed as the result of a   generalized McDiarmid method (e.g.,   \cite{kontorovich2014concentration} in the sub-Gaussian case), applied to  the family of functions $x\mapsto E_{P_Y}[g(x,\cdot)]$, $g\in\mathcal{G}$. In other words, for large $m$ our result approaches the bound on the corresponding problem with the $Y$-dependence integrated out.  Thus our result captures the benefit of sample reuse, matching the  motivating discussion from Section \ref{sec:intro}.
\end{remark}
 \begin{proof}
 Note that the assumptions imply $h_{0,\mathcal{Y}},k_{\mathcal{Y}}\in L^1(P_Y^2)$ and $r_{\mathcal{X}}\in L^1(P_X)$.  Therefore $h_{\mathcal{Y}}\in L^1(P_X\times P_Y\times P_Y)$ and
 \begin{align}
 \int H(x_0,y_0)(P_X\times P_Y)(dx_0dy_0)<\infty\,,
 \end{align}
 where
 \begin{align}
 H(x_0,y_0)\coloneqq \int(h_{\mathcal{X}}(x,x_0,y_0)+h_{\mathcal{Y}}(x,y,y_0)+h_{0,\mathcal{Y}}(y,y_0)+k_{\mathcal{Y}}(y,y_0))(P_X\times P_Y)(dxdy)\,.
 \end{align}
 Hence there exists $x_0\in\mathcal{X}$ and $y_0\in\mathcal{Y}$ such that  $h_{0,\mathcal{Y}}(\cdot,y_0)\in L^1(P_Y)$,  $k_{\mathcal{Y}}(\cdot,y_0)\in L^1(P_Y)$,  $h_{\mathcal{X}}(\cdot,x_0,y_0)\in L^1(P_X)$, and $h_{\mathcal{Y}}(\cdot,\cdot,y_0)\in L^1(P_X\times P_Y)$.

Now compute the bounds
 \begin{align}
 |g(x,y)|\leq &|g(x,y)-g(x,y_0)|+|g(x,y_0)-g(x_0,y_0)|+|g(x_0,y_0)|\\
 \leq& {h}_{\mathcal{Y}}(x,y,y_0)+h_{\mathcal{X}}(x,x_0,y_0)+|g(x_0,y_0)|\notag
 \end{align}
 and
  \begin{align}
   |g(\tilde{x},\tilde{y})-g(x,y)|\leq &    |g(\tilde{x},\tilde{y})-g(x_0,y_0)|+|g(x_0,y_0)-g(x,y)|\\
   \leq&  |g(\tilde{x},\tilde{y})-g(\tilde{x},y_0)|+|g(\tilde{x},y_0)-g(x_0,y_0)|\notag\\
   &+|g(x,y)-g(x,y_0)|+|g(x,y_0)-g(x_0,y_0)|\notag\\
      \leq&    
      h_{\mathcal{Y}}(\tilde{x},\tilde{y},y_0)+h_{\mathcal{X}}(\tilde{x},x_0,y_0)  
      +h_{\mathcal{Y}}(x,y,y_0) +h_{\mathcal{X}}(x,x_0,y_0)\,.\notag
  \end{align}
  Therefore $ \mathcal{G}\subset L^1(P_X\times P_Y)$ and
  \begin{align}
  |\phi_\pm(z)|\leq &\sup_{g\in\mathcal{G}}\left| \frac{1}{nm}\sum_{i=1}^n\sum_{j=1}^m g(x_i,y_{i,j})-E_{P_X\times P_Y}[g]\right|\\
 \leq& \sup_{g\in\mathcal{G}}\frac{1}{nm}\sum_{i=1}^n\sum_{j=1}^m E_{(x,y)\sim P_X\times P_Y}[\left|g(x_i,y_{i,j})-g(x,y)\right|]\notag\\
  \leq& \frac{1}{nm}\sum_{i=1}^n\sum_{j=1}^m \big( h_{\mathcal{Y}}(x_i,y_{i,j},y_0)+h_{\mathcal{X}}(x_i,x_0,y_0)  \notag\\
  &\qquad\qquad\qquad\qquad          +E_{P_X\times P_Y}[ h_{\mathcal{Y}}(\cdot,\cdot,y_0)] +E_{P_X}[ h_{\mathcal{X}}(\cdot,x_0,y_0) ]\big)<\infty\,.\notag
  \end{align}
 
  Hence  
     $\phi_\pm$ is real-valued and
    \begin{align}
    &\mathbb{E}[|\phi_\pm(Z)|]
    \leq  2 E_{P_X\times P_Y}[h_{\mathcal{Y}}(\cdot,\cdot,y_0)]+2E_{P_X}[h_{\mathcal{X}}(\cdot,x_0,y_0)]  <\infty\,.
    \end{align}
    Therefore $\phi_t(Z)\in L^1(\mathbb{P})$ for all $t\in[0,1]$.

Next we compute
 \begin{align}
   &\mathbb{E}[|\phi_\pm(z^{1:\ell},Z^{\ell+1:n})|]\\
   \leq &\frac{1}{nm}\sum_{i=1}^\ell\sum_{j=1}^m \left( h_{\mathcal{Y}}(x_i,y_{i,j},y_0)+h_{\mathcal{X}}(x_i,x_0,y_0)  
              +E_{P_X\times P_Y}[ h_{\mathcal{Y}}(\cdot,\cdot,y_0)] +E_{P_X}[ h_{\mathcal{X}}(\cdot,x_0,y_0) ]\right)\notag\\
              &    
              +2\frac{n-\ell}{n} \left( 
                            E_{P_X\times P_Y}[ h_{\mathcal{Y}}(\cdot,\cdot,y_0)] +E_{P_X}[ h_{\mathcal{X}}(\cdot,x_0,y_0) ]\right)<\infty\notag
              \end{align}

and
\begin{align}
&\mathbb{E}[|\phi_\pm(z^{1:i^\prime-1},x_{i^\prime},y_{i^\prime}^{1:\ell},Y_{i^\prime}^{\ell+1:m},Z^{i^\prime+1:n})|]\\
\leq&  \frac{1}{nm}\sum_{i=1}^{i^\prime-1}\sum_{j=1}^m \left( h_{\mathcal{Y}}(x_i,y_{i,j},y_0)+h_{\mathcal{X}}(x_i,x_0,y_0)  
+E_{P_X\times P_Y}[ h_{\mathcal{Y}}(\cdot,\cdot,y_0)] +E_{P_X}[ h_{\mathcal{X}}(\cdot,x_0,y_0) ]\right)\notag\\              
&+\frac{1}{nm}\sum_{j=1}^{\ell} \left( h_{\mathcal{Y}}(x_{i^\prime},y_{i^\prime,j},y_0)+h_{\mathcal{X}}(x_{i^\prime},x_0,y_0)  
+E_{P_X\times P_Y}[ h_{\mathcal{Y}}(\cdot,\cdot,y_0)] +E_{P_X}[ h_{\mathcal{X}}(\cdot,x_0,y_0) ]\right)\notag\\
&                           
+\frac{1}{nm}\sum_{j=\ell+1}^m \left( {E}_{P_Y}[h_{0,\mathcal{Y}}(\cdot,y_0)]+r_{\mathcal{X}}(x_{i^\prime}){E}_{P_Y}[k_{\mathcal{Y}}(\cdot,y_0)]\right)\notag\\
&+\frac{1}{nm}\sum_{j=\ell+1}^m \left(h_{\mathcal{X}}(x_{i^\prime},x_0,y_0)                                                         +E_{P_X\times P_Y}[ h_{\mathcal{Y}}(\cdot,\cdot,y_0)] +E_{P_X}[ h_{\mathcal{X}}(\cdot,x_0,y_0) ]\right)\notag\\
& 
+2\frac{n-i^\prime}{n}  \left( E_{P_X\times P_Y}[h_{\mathcal{Y}}(\cdot,\cdot,y_0)]+E_{P_X}[h_{\mathcal{X}}(\cdot,x_0,y_0)]  
\right)<\infty\,.\notag
\end{align}
 Therefore $ \phi_t(z^{1:\ell},Z^{\ell+1:n}),\phi_t(z^{1:i^\prime-1},x_{i^\prime},y_{i^\prime}^{1:\ell},Y_{i^\prime}^{\ell+1:m},Z^{i^\prime+1:n})\in L^1(\mathbb{P})$.

  Now we compute one-component  mean-difference bounds for $\mathcal{X}$:
   \begin{align}
   &\left|\mathbb{E}[\phi_\pm(z^{1:i-1},x_i,Y_i^{1:m},Z^{i+1:n})]-\mathbb{E}[\phi_\pm(z^{1:i-1},\tilde{x}_i,Y_i^{1:m},Z^{i+1:n})]\right|\\
 \leq   &\mathbb{E}[\left|\phi_\pm(z^{1:i-1},x_i,Y_i^{1:m},Z^{i+1:n})-\phi_\pm(z^{1:i-1},\tilde{x}_i,Y_i^{1:m},Z^{i+1:n})\right|]\notag\\
   \leq&\mathbb{E}\left[\sup_{g\in\mathcal{G}}\left|\frac{1}{nm} \sum_{j=1}^m g(x_i,Y_{i,j})-\frac{1}{nm} \sum_{j=1}^m g(\tilde{x}_i,Y_{i,j})\right|\right]\notag\\
   \leq&\frac{1}{n} {E}_{P_Y}\left[ h_{\mathcal{X}}(x_i,\tilde{x}_i,\cdot)\right]\coloneqq h_{\mathcal{X}_i}(x_i,\tilde{x}_i)\notag\,.
   \end{align}
Hence we also have
\begin{align}
   &\left|\mathbb{E}[\phi_t(z^{1:i-1},x_i,Y_i^{1:m},Z^{i+1:n})]-\mathbb{E}[\phi_t(z^{1:i-1},\tilde{x}_i,Y_i^{1:m},Z^{i+1:n})]\right|
 \leq h_{\mathcal{X}_i}(x_i,\tilde{x}_i)\,.
\end{align}

Next we obtain one-component  mean-difference bounds for $\mathcal{Y}$:
\begin{align}
 &\left|\mathbb{E}[\phi_\pm(z^{1:i-1},x_i,y_i^{1:j},Y_i^{j+1:m},Z^{i+1:n})]-\mathbb{E}[\phi_\pm(z^{1:i-1},x_i,y_i^{1:j-1},\tilde{y}_{i,j},Y_i^{j+1:m},Z^{i+1:n})]\right|\\
 \leq&   \mathbb{E}[|\phi_\pm(z^{1:i-1},x_i,y_i^{1:j},Y_i^{j+1:m},Z^{i+1:n})-\phi_\pm(z^{1:i-1},x_i,y_i^{1:j-1},\tilde{y}_{i,j},Y_i^{j+1:m},Z^{i+1:n})|]\notag\\
 \leq&\sup_{g\in\mathcal{G}}\left|\frac{1}{nm}  g(x_i,y_{i,j})
 -
 \frac{1}{nm} g(x_i,\tilde{y}_{i,j}) \right|\notag\\
 \leq&\frac{1}{nm}   h_{\mathcal{Y}}(x_i,y_{i,j},\tilde{y}_{i,j})\coloneqq h_{\mathcal{Y}_{i,j}}(x_i,y_{i,j},\tilde{y}_{i,j})\,.\notag
 \end{align}
  Hence  
 \begin{align}
  &\left|\mathbb{E}[\phi_t(z^{1:i-1},x_i,y_i^{1:j},Y_i^{j+1:m},Z^{i+1:n})]-\mathbb{E}[\phi_t(z^{1:i-1},x_i,y_i^{1:j-1},\tilde{y}_{i,j},Y_i^{j+1:m},Z^{i+1:n})]\right|\\
  \leq&  h_{\mathcal{Y}_{i,j}}(x_i,y_{i,j},\tilde{y}_{i,j})\,.\notag
 \end{align}
 
 Finally, we compute
 \begin{align}
 \int \cosh(\lambda h_{\mathcal{X}_i}) dP_X^2
 =& \int \cosh(n^{-1}\lambda E_{P_Y}[h_{\mathcal{X}}(x,\tilde{x},\cdot)])P_X^2(dxd\tilde{x})\\
 \leq & \exp(\xi_{\mathcal{X}}(|\lambda|/n))\,\,\,\,\text{ for all }|\lambda|<n K_{\mathcal{X}}\,.\notag
 \end{align}

 Therefore  the assumptions of  Theorem \ref{thm:ULLN_bound_non_unif_MGF_Phi1_Phi2_case} hold for $\phi_t$ with 
 \begin{align}
 \xi_{\mathcal{X}_i}(\lambda)\coloneqq\begin{cases}
 \xi_{\mathcal{X}}(\lambda/n) &\text{ if }0\leq \lambda<nK_{\mathcal{X}}\\
 \infty&\text{ otherwise}
 \end{cases}\,,
 \end{align}
 $h_{0,\mathcal{Y}_{i,j}}\coloneqq \frac{1}{nm}h_{0,\mathcal{Y}}$, $r_{\mathcal{X}_i}\coloneqq \frac{1}{nm}r_{\mathcal{X}}$, $k_{\mathcal{Y}_{i,j}}\coloneqq k_{\mathcal{Y}}$. After substituting the relevant quantities into \eqref{eq:ULLN_bound_non_unif_MGF_Phi1_Phi2_case},  changing variables from $\lambda$ to $n\lambda$, and then simplifying, we obtain
 \begin{align}\label{eq:ULLN_non_uniform_mean_diff_bounds_temp}
      &\mathbb{P}\left(\pm(\phi_t(Z)-\mathbb{E}[\phi_t(Z)])\geq \epsilon\right)\\
      \leq&  \exp\left(-n\!\sup_{\lambda\in[ 0,\tilde{\lambda}_m)}\!\left\{\lambda \epsilon- p^{-1} 
       \xi_{\mathcal{X}}(p\lambda)  
      - m^{-1}\left(\frac{ 2(\alpha \lambda)^2}{1-m^{-1}q(\alpha\lambda)^2}+\frac{2(\beta\lambda )^{2}}{1- ( m^{-1}\beta\lambda)^{2}}\right)\right\}\right)\notag\,,\\
      &\tilde{\lambda}_m\coloneqq \min\{m\beta^{-1},(m^{-1}q\alpha^2)^{-1/2},p^{-1}K_\mathcal{X}\}\,,\notag
    \end{align}
 for any choice of conjugate exponents $p,q\in(1,\infty)$.  In order to ensure the upper bound approaches (as $m\to\infty$) the corresponding result with the $Y$-dependence integrated out, we choose $\rho>0$   such that $\rho m>1$ and apply  the above to $q=\rho m$, $p=q/(q-1)=(1-(\rho m)^{-1})^{-1}$; after further simplification, this yields  the claimed result.
 \end{proof}

 \subsection{Concentration Inequalities for Stochastic Optimization Problems}\label{sec:stoch_opt}

 Here we  build on Theorem \ref{thm:ULLN_non_uniform_mean_diff_bounds}  to obtain concentration inequalities for stochastic optimization problems  which capture the benefit of sample reuse and allow for non-uniform MGF bounds. Our goal is to bound the error inherent in solving   the empirical problem,
 \begin{align}\label{eq:stoch_opt_emp_prob}
 \theta^*_{n,m}\in\operatorname*{arg\,min}_{\theta\in\Theta}\frac{1}{nm}\sum_{i=1}^n\sum_{j=1}^m g_\theta(X_i,Y_{i,j})\,,
 \end{align}
 in place of the population problem, $\inf_\theta E_{P_X\times P_Y}[g_\theta]$. More specifically, we will derive concentration inequalities for the excess statistical error
 \begin{align}\label{eq:stoch_opt_stat_error}
  E_{P_X\times P_Y}[g_{\theta^*_{n,m}}]- \inf_{\theta\in\Theta} E_{P_X\times P_Y}[g_\theta]\,.
 \end{align}

 We will require the following additional assumptions (which also allow $\theta^*_{n,m}$ to only approximately solve the empirical problem \eqref{eq:stoch_opt_emp_prob}).
 \begin{assumption}\label{assump:stoch_opt_nonuniform_bound}
 Let $(\Theta,d_\Theta)$ be a nonempty separable  pseudometric space equipped with a $\sigma$-algebra, $(\mathcal{X},P_X)$, $(\mathcal{Y},P_Y)$ be   probability spaces, and $g:\Theta\times \mathcal{X}\times\mathcal{Y}\to\mathbb{R}$ be measurable.  Suppose we have $c:\Theta\to\mathbb{R}$ and $h\in L^1(P_X\times P_Y)$ such that $\theta\mapsto g_\theta(x,y)-c_\theta$ is sequentially continuous for all $x\in\mathcal{X}$, $y\in\mathcal{Y}$ and $|g_\theta-c_\theta|\leq h$ for all $\theta\in\Theta$.
 
   Suppose we also have a probability space $(\Omega,\mathbb{P})$ and i.i.d. $\mathcal{X}$-valued random variables $X_i$, $i=1,\dots,n$, distributed as $P_X$ and i.i.d. $\mathcal{Y}$-valued random variables $Y_{i,j}$, $i=1,\dots,n$, $j=1,\dots,m$, distributed as $P_Y$; suppose the $X_i$, $Y_{i,j}$ are all mutually independent and define $Z$ as in \eqref{eq:Z_def}.  Suppose there exists $f\in L^1(P_X\times P_Y)$ such that $g_\theta\geq f$ for all $\theta\in\Theta$.  Finally, let $\epsilon_{n,m}^{\mathrm{opt}}\geq 0$ and suppose we have a $\Theta$-valued random variable $\theta^*_{n,m}$ that satisfies
   \begin{align}
   \frac{1}{nm}\sum_{i=1}^n\sum_{j=1}^m g_{\theta^*_{n,m}}(X_i,Y_{i,j})\leq \inf_{\theta\in\Theta}\frac{1}{nm}\sum_{i=1}^n\sum_{j=1}^m g_{\theta}(X_i,Y_{i,j})+\epsilon^{\mathrm{opt}}_{n,m}\,\,\,\,\text{$\mathbb{P}$-a.s.}
   \end{align}
 \end{assumption}
 \begin{remark}
 We define separability of a pseudometric space to mean the existence of a   countable dense subset, i.e., there exists a countable subset  $\Theta_0\subset\Theta$ such that for all $\theta\in\Theta$ there exists a sequence $\theta_j\in\Theta_0$ satisfying $d_\Theta(\theta_j,\theta)\to 0$. We allow for a nonzero $c_\theta$ in order to cancel a possible parameter that shifts the output of $g_\theta$  (e.g.,  involving the bias parameter of the output layer of a neural network), which would otherwise preclude the existence of an appropriate envelope function, $h$.
 \end{remark}

 We will now derive a concentration inequality for the excess statistical error \eqref{eq:stoch_opt_stat_error}.
 \begin{theorem}\label{thm:stoch_opt}
  Suppose we have a nonempty pseudometric space  $(\Theta,d_\Theta)$,  probability spaces $(\mathcal{X},P_X)$ and $(\mathcal{Y},P_Y)$, and functions $g_\theta:\mathcal{X}\times\mathcal{Y}\to\mathbb{R}$, $\theta\in\Theta$, that 
  satisfy  Assumption \ref{assump:ULLN} (with $\mathcal{G}\coloneqq\{g_\theta:\theta\in\Theta\}$) and Assumption \ref{assump:stoch_opt_nonuniform_bound}.

      Let    $s\in(0,1)$ and   $\rho>0$   such that $\rho m>1$.  Then   for all $\epsilon>0$ we have
\begin{align}
&\mathbb{P}\left(E_{P_X\times P_Y}[g_{\theta^*_{n,m}}]- \inf_\theta E_{P_X\times P_Y}[g_\theta]\geq \epsilon+2\mathbb{E}\left[\phi_{1/2}(Z)\right]+\epsilon_{n,m}^{\mathrm{opt}}\right)\\
        \leq&  \exp\left(-n\!\sup_{\lambda\in[ 0,\lambda_{m})}\!\left\{\lambda \epsilon/2-\xi_{\mathcal{X},m}(\lambda)\right\}\right)\,,\notag
      \end{align}  
  where $\xi_{\mathcal{X},m}$ was defined in \eqref{eq:xi_Xm_def}, $\lambda_m$ was defined in \eqref{eq:lambda_m_def}, $\alpha$ and $\beta$ were defined in \eqref{eq:alpha_beta_def}, and $\phi_{1/2}$ is as defined in Theorem \ref{thm:ULLN_non_uniform_mean_diff_bounds}.

 \end{theorem}
 \begin{remark}\label{remark:StochOptlarge_m_limit} 
Essentially the same discussion as in Remark \ref{remark:ULLNlarge_m_limit}  applies here, and so our technique captures the benefit of sample reuse in this case as well.
 \end{remark}
 \begin{proof}
 
   We have the $\mathbb{P}$-a.s. error decomposition
     \begin{align}\label{eq:as_opt_bound}
     0\leq &E_{P_X\times P_Y}[g_{\theta^*_{n,m}}]- \inf_\theta E_{P_X\times P_Y}[g_\theta]\\
     \leq&E_{P_X\times P_Y}[g_{\theta^*_{n,m}}]-\frac{1}{nm}\sum_{i=1}^n \sum_{j=1}^m g_{\theta^*_{n,m}}(X_i,Y_{i,j})\notag\\
     &+\inf_{\theta\in\Theta}\frac{1}{nm}\sum_{i=1}^n \sum_{j=1}^m g_{\theta}(X_i,Y_{i,j})- \inf_\theta E_{P_X\times P_Y}[g_\theta]+\epsilon_{n,m}^{\text{opt}}\notag\\
     \leq& \phi_{+}(Z)+\phi_-(Z)+\epsilon_{n,m}^{\text{opt}}\,,\notag
     \end{align}
     where $\phi_\pm$ is as defined in \eqref{eq:phi_pm_def}; we note that the infima are finite due to the assumed lower bound on $g$.

     Let  $\Theta_0$ be a countable dense subset of $\Theta$. Using the dominated convergence theorem (applied to $g_{\theta}-c_\theta$ along appropriate sequences in $\Theta_0$), we obtain
       \begin{align}\label{eq:phi_pm_Theta0}
       \phi_\pm(z)
       =&\sup_{\theta \in\Theta_0}\left\{\pm\left(\frac{1}{nm}\sum_{i=1}^n\sum_{j=1}^m g_\theta(x_i,y_{i,j})-E_{P_X\times P_Y}[g_\theta]\right)\right\}\,.
       \end{align}   
   Therefore we can compute   
\begin{align}
&\mathbb{P}\left(E_{P_X\times P_Y}[g_{\theta^*_{n,m}}]- \inf_\theta E_{P_X\times P_Y}[g_\theta]\geq \epsilon+2\mathbb{E}[\phi_{1/2}(Z)]+\epsilon_{n,m}^{\mathrm{opt}}\right)\\
\leq&\mathbb{P}\left(\phi_{1/2}(Z)-\mathbb{E}[\phi_{1/2}(Z)]\geq \epsilon/2\right)\,,\notag
      \end{align}
      where $\phi_{1/2}\coloneqq (\phi_++\phi_-)/2$.
     Theorem \ref{thm:ULLN_non_uniform_mean_diff_bounds}  applied to $\phi_{1/2}$ then yields the claimed result.
  \end{proof}

  \subsection{ULLN Mean Bounds}\label{sec:mean_bound}
  In this section we will present tools that lead to useful bounds on the expected value of $\phi_t(Z)$, which appears in Theorems \ref{thm:ULLN_non_uniform_mean_diff_bounds} and \ref{thm:stoch_opt}.  We will utilize standard Rademacher complexity and Dudley's entropy integral techniques, though some novelty is required in order to obtain bounds which  meet the following goals. The bounds should: 1)  capture the benefits of sample reuse; as $m\to\infty$, they should converge to bounds involving only the Rademacher complexity of the function class $x\mapsto E_{P_Y}[g_\theta(x,\cdot)]$, $\theta\in\Theta$, 2) apply to unbounded functions and distributions with unbounded support, 3) address the curse of dimensionality by capturing the benefits of $P_X$ having low-dimensional support, even when $P_Y$ does not.
  \begin{lemma}\label{lemma:ULLN_mean_decomp}
   Working under  Assumption \ref{assump:ULLN2}, and with $\phi_\pm$ as defined in \eqref{eq:phi_pm_def}, we have
  \begin{align}\label{eq:ULLN_mean_decomp}
  \mathbb{E}[\phi_\pm(Z)]
  \leq& 2\mathcal{R}_{\mathcal{G}_{P_Y},P_X,n}+ 2  E_{P_X^n\times P_Y^{nm}}\left[ \widehat{\mathcal{R}}_{\mathcal{G},nm}\circ H\right] \,,
  \end{align}
   where
     \begin{align}
     \mathcal{G}_{P_Y}\coloneqq \{x\in\mathcal{X}\mapsto E_{P_Y}[g(x,\cdot)]:g\in\mathcal{G}\}
     \end{align}
     and $H:\mathcal{X}^n\times\mathcal{Y}^{nm}\to (\mathcal{X}\times\mathcal{Y})^{nm}$ is defined by
   \begin{align}
   H(x,y)\coloneqq ((x_1,y_{1,1}),\dots,(x_1,y_{1,m}),\dots,(x_n,y_{n,1}),\dots,(x_n,y_{n,m}))\,.
   \end{align}
   \end{lemma}
   \begin{remark}\label{remark:Rademacher_large_m}
   Our definitions of the empirical and mean Rademacher complexities, $\widehat{\mathcal{R}}_{\mathcal{G},nm}$ and $\mathcal{R}_{\mathcal{G}_{P_Y},P_X,n}$ respectively, can be found  in Remark \ref{remark:Rademacher_def} in Appendix \ref{lemma:ULLN_non-identical}.
   For many choices of $\mathcal{G}$, standard techniques can be used to show that 
    \begin{align}
    E_{P_X^n\times P_Y^{nm}}\left[ \widehat{\mathcal{R}}_{\mathcal{G},nm}\circ H\right] \to 0\,\,\,\,\text{ as }m\to\infty\,,
    \end{align}
    in which case the right-hand side of \eqref{eq:ULLN_mean_decomp} converges to the standard ULLN bound for the family of functions $ \mathcal{G}_{P_Y}$, i.e., for the original function class, but with the $y$-dependence integrated out.  Thus our result captures the benefit of sample reuse as desired. 
    \end{remark}
   \begin{proof}
  The integrability assumptions imply the existence of $y_0\in\mathcal{Y}$ such that $h_{0,\mathcal{Y}}(\cdot,y_0),k_{\mathcal{Y}}(\cdot,y_0)\in L^1(P_Y)$. Hence for all $x\in\mathcal{X}$ we have
   \begin{align}
   E_{P_Y(dy)}[|g(x,y)|]\leq & E_{P_Y(dy)}[|g(x,y)-g(x,y_0)|]+|g(x,y_0)|\\
   \leq&E_{P_Y(dy)}[h_{0,\mathcal{Y}}(y,y_0)]+r_{\mathcal{X}}(x)E_{P_Y(dy)}[k_{\mathcal{Y}}(y,y_0) ]+|g(x,y_0)|<\infty\,.\notag
   \end{align}
   Therefore $g(x,\cdot)\in L^1(P_Y)$ for all $x\in\mathcal{X}$. 
     Based on this, we can bound
   \begin{align}
   \mathbb{E}[\phi_\pm(Z)]
   \leq&\mathbb{E}\left[\sup_{g\in\mathcal{G}}\left\{\pm\left(\frac{1}{n}\sum_{i=1}^n  (E_{P_Y}[g(X_i,\cdot)]-E_{P_X\times P_Y}[g])\right)\right\}\right]\\
   &+ \mathbb{E}\left[\sup_{g\in\mathcal{G}}\left\{\pm\left(\frac{1}{nm}\sum_{i=1}^n\sum_{j=1}^m (g(X_i,Y_{i,j})-E_{P_Y}[g(X_i,\cdot)])\right)\right\}\right]\,.\notag
   \end{align}
  
  Considering the first term, we have
  \begin{align}
   &\mathbb{E}\left[\sup_{g\in\mathcal{G}}\left\{\pm\left(\frac{1}{n}\sum_{i=1}^n  (E_{P_Y}[g(X_i,\cdot)]-E_{P_X\times P_Y}[g])\right)\right\}\right]\\
   =&\mathbb{E}\left[\sup_{\tilde{g}\in\mathcal{G}_{P_Y}}\left\{\pm\left(\frac{1}{n}\sum_{i=1}^n  (\tilde{g}(X_i)-E_{P_X}[\tilde{g}])\right)\right\}\right]\notag\\
    \leq &   2E_{P_X^n}\left[ \widehat{\mathcal{R}}_{\mathcal{G}_{P_Y},n}\right]= 2\mathcal{R}_{\mathcal{G}_{P_Y},P_X,n}\,,\notag
   \end{align}
  where we used a standard ULLN bound (e.g., a special case of  Lemma \ref{lemma:ULLN_non-identical}).
  
  To bound the second term, fix $x\in\mathcal{X}^n$ and define the probability measures $\mu_{x,i,j}\coloneqq \delta_{x_i}\times P_Y$ on $\mathcal{X}\times\mathcal{Y}$ and let $\mu_x^{n,m}\coloneqq \prod_{i=1}^n\prod_{j=1}^m \mu_{x,i,j}$. We have $\mathcal{G}\subset L^1(\mu_{x,i,j})$ for all $i,j$, and therefore   Lemma \ref{lemma:ULLN_non-identical} implies (here the ability to handle the non-identical case is crucial)
   \begin{align}
   &E_{\mu_x^{n,m}}\left[\sup_{g\in\mathcal{G}}\left\{\pm\left(\frac{1}{nm}\sum_{i=1}^n\sum_{j=1}^m \left(g(z_{i,j})-E_{\mu_{x,i,j}}[g]\right)\right)\right\}\right]
   \leq    2E_{\mu_x^{n,m}}\left[ \widehat{\mathcal{R}}_{\mathcal{G},nm}\right]\,,
   \end{align}
   i.e.,
    \begin{align}
    &E_{P_Y^{nm}}\left[\sup_{g\in\mathcal{G}}\left\{\pm\left(\frac{1}{nm}\sum_{i=1}^n\sum_{j=1}^m \left(g(x_i,y_{i,j})-E_{P_Y}[g(x_i,\cdot)]\right)\right)\right\}\right]\\
    \leq&    2E_{P_Y^{nm}(dy)}\left[ \widehat{\mathcal{R}}_{\mathcal{G},nm}(H(x,y))\right]\,.\notag
    \end{align}
    Therefore
    \begin{align}
    &\mathbb{E}\left[\sup_{g\in\mathcal{G}}\left\{\pm\left(\frac{1}{nm}\sum_{i=1}^n\sum_{j=1}^m (g(X_i,Y_{i,j})-E_{P_Y}[g(X_i,\cdot)])\right)\right\}\right]\\
    =&\int {E}_{ P_Y^{nm}}\left[\sup_{g\in\mathcal{G}}\left\{\pm\left(\frac{1}{nm}\sum_{i=1}^n\sum_{j=1}^m (g(x_i,y_{i,j})-E_{P_Y}[g(x_i,\cdot)])\right)\right\}\right]P_X^n(dx)\notag\\
    \leq& 2 E_{P_X^n\times P_Y^{nm}}\left[ \widehat{\mathcal{R}}_{\mathcal{G},nm}\circ H\right] \,.\notag
    \end{align}
    Combining these we obtain the claimed result. 
   \end{proof}
   
   Lemma \ref{lemma:ULLN_mean_decomp} clearly illustrates the benefits of sample reuse, and the corresponding reduction to a Rademacher complexity bound with the $y$-dependence integrated out as $m\to\infty$.  However, for use in applications we will need to analyze each term further. Our handling of the first term in the bound \eqref{eq:ULLN_mean_decomp} will be structured so as to leverage the benefits of  $P_X$ having low-dimensional (but unbounded)   support.  For the second term, we will use bounds on parametric function classes; this is appropriate as we are able to utilize as many $Y$-samples as desired, thus $m$ can always be chosen large enough to move the bound into the underparameterized regime (relative to $nm$).  In underparameterized settings (relative to $n$), one could also treat the first term via parametric bounds, but the present approach is generally more appropriate for  high-dimensional applications, including our study of DSM in Section \ref{sec:DSM}. For these purposes, we will require the following additional assumptions.
   \begin{assumption}\label{assump:ULLN_parameterized}
   Assume the conditions of Assumption \ref{assump:ULLN2} hold for $\mathcal{G}=\{g_\theta:\theta\in\Theta\}$, where  $(\Theta,d_\Theta)$ is  a countable  pseudometric space and   $g_\theta: \mathcal{X}\times\mathcal{Y}\to\mathbb{R}$, $\theta\in\Theta$. Suppose also that 
we have $c:\Theta\to\mathbb{R}$ and a measurable $L:\mathcal{X}\times\mathcal{Y}\to (0,\infty)$ such that  $\theta\mapsto g_\theta(x,y)-c_\theta$ is $L(x,y)$-Lipschitz  (w.r.t. $(d_\Theta,|\cdot|)$) for all $x\in\mathcal{X}$, $y\in\mathcal{Y}$. 
\end{assumption}
    \begin{lemma}\label{lemma:mean_bound}
    Under Assumption \ref{assump:ULLN_parameterized},  define 
        \begin{align}\label{eq:tilde_G_defs}
        \tilde{\mathcal{G}}\coloneqq &\{g_\theta-c_\theta:\theta\in\Theta\}\,,\,\,\,\,\,
                        \tilde{\mathcal{G}}_{P_Y}\coloneqq \{x\in\mathcal{X}\mapsto E_{P_Y}[\tilde{g}(x,\cdot)]:\tilde{g}\in\tilde{\mathcal{G}}\}\,,
              \end{align}
              and suppose   the following:
            \begin{enumerate}      
        \item We have a measurable   $h:\mathcal{X}\to[0,\infty]$   such that $|\tilde{g}|\leq h$ for all $\tilde{g}\in\tilde{\mathcal{G}}_{P_Y}$. 
        \item We have  $a,b\in\mathbb{R}$ and a measurable function $\psi:\mathcal{X}\to[0,1]$ such that $a\leq \psi \tilde{g}\leq b$ for all $\tilde{g}\in\tilde{\mathcal{G}}_{P_Y}$.
        \item We have a measurable set $\mathcal{X}_*\subset\mathcal{X}$ with $P_X(\mathcal{X}_*^c)=0$.
        \item  We have a measurable set $B\subset\mathcal{X}$ with $\{\psi\neq 0\}\subset B$ and either $\mathcal{X}_*\subset B$ or there exists   $x^\prime\in \mathcal{X}_*\cap B$ such that $\psi(x^\prime)=0$.
        \end{enumerate}
    Then with $\phi_\pm$ as defined in \eqref{eq:phi_pm_def}, we have
      \begin{align}\label{eq:mean_bound}
        \mathbb{E}[\phi_\pm(Z)]\leq & 2E_{P_X}[(1-\psi)h]\\
        &+ 2 \left(   E_{P_X\times P_Y}[ L^2] \right)^{1/2}        \inf_{\delta>0}\left\{4\delta +1_{\delta<D_\Theta}8\sqrt{2}(nm)^{-1/2} \int_{\frac{1}{2}\delta}^{\frac{1}{2}D_\Theta} \sqrt{\log N(r,\Theta,d_\Theta)} dr\right\}\notag\\
        & +2 \inf_{\delta>0}\left\{4\delta +1_{\delta<b-a}8\sqrt{2}n^{-1/2}\int_{\frac{1}{2}\delta}^{\frac{1}{2}(b-a)}\sqrt{\log N(r,(\psi\tilde{\mathcal{G}}_{P_Y})|_{B\cap\mathcal{X}_*} ,\|\cdot\|_\infty)} dr\right\}\,,\notag
        \end{align} 
    where $D_\Theta$ is the diameter of $\Theta$, $N(r,M,d_M)$ denotes the radius-$r$ covering number of the space $M$, with respect to the pseudometric $d_M$,  and  
          \begin{align}
          \psi\tilde{\mathcal{G}}_{P_Y}\coloneqq\{\psi \tilde{g}:\tilde{g}\in \tilde{\mathcal{G}}_{P_Y}\}\,.
          \end{align}
   
   Moreover, if $\Theta$ is a countable dense subset of the unit ball (under some norm) in $\mathbb{R}^k$ and $d_\Theta$ is the corresponding metric then 
   \begin{align}\label{eq:Dudley_parametric_unit_ball}
   &\inf_{\delta>0}\left\{4\delta +1_{\delta<D_\Theta}8\sqrt{2}(nm)^{-1/2} \int_{\frac{1}{2}\delta}^{\frac{1}{2}D_\Theta} \sqrt{\log N(r,\Theta,d_\Theta)} dr\right\}   \leq \left(8\sqrt{2}\int_{0}^{1}\sqrt{\log (1+2/r)} dr\right)\left(\frac{k}{nm}\right)^{1/2}\,.
   \end{align}

   \end{lemma}
    \begin{remark}
    Regarding the above bound, we note the following key points: 1)
  The first term on the right-hand side is only needed when dealing with non-compact $\mathcal{X}_*$; in the compact case one can simply take $\psi=1$ and $B=\mathcal{X}$.
   2)   Assuming $L\in L^2(P_X\times P_Y)$, if the parameter space $\Theta$ is a countable subset of a ball in $\mathbb{R}^k$ then one obtains an $O(\sqrt{k/(nm)})$ bound on the second term; even if $n$ is  small compared to $k$ (i.e., the overparameterized regime relative to the size of the training set), this term can be made arbitrarily small by sufficient sample reuse (i.e., large $m$).    3) In the last term, the entropy integral depends on the function class restricted to $B\cap\mathcal{X}_*$, and so when $\mathcal{X}_*$ has   intrinsic dimension $d_*$  one can  obtain bounds  whose scaling is governed by $d_*$ instead of $d$ (the dimension of $\mathcal{X}$) in a number of important cases by using the tools from \cite{kolmogorov1959varepsilon};  we further develop this in Theorem \ref{thm:mean_bound_smooth} and  illustrate its use in Section \ref{sec:DSM}.
    \end{remark}
   \begin{proof}
   Applying Lemma \ref{lemma:ULLN_mean_decomp} to $\tilde{\mathcal{G}}$ (and canceling the $c_\theta$'s on the left-hand side), we get
   \begin{align}\label{eq:mean_bound_temp}
   \mathbb{E}[\phi_\pm(Z)]
   \leq& 2\mathcal{R}_{\tilde{\mathcal{G}}_{P_Y},P_X,n}+ 2  E_{P_X^n\times P_Y^{nm}}\left[ \widehat{\mathcal{R}}_{\tilde{\mathcal{G}},nm}\circ H\right] \,.
   \end{align}
   An   application of Dudley's entropy integral for a Lipschitz parameterized function class gives (specifically, we use Theorem   3.5 from \cite{doi:10.1137/25M1748664})
     \begin{align}
     \widehat{\mathcal{R}}_{\tilde{\mathcal{G}},nm}(z)\leq & L_{nm}(z)\inf_{\delta>0}\left\{4\delta +1_{\delta<D_\Theta}8\sqrt{2}(nm)^{-1/2} \int_{\frac{1}{2}\delta}^{\frac{1}{2}D_\Theta} \sqrt{\log N(r,\Theta,d_\Theta)} dr\right\}\,,\\  
     L_{nm}(z)\coloneqq& \left(\frac{1}{nm}\sum_{i=1}^{n}\sum_{j=1}^mL(z_{i,j})^2\right)^{1/2}\,.\notag
     \end{align}
   Taking expected values and using the Cauchy-Schwarz inequality, we find
     \begin{align}
    &  E_{P_X^n\times P_Y^{nm}}\left[ \widehat{\mathcal{R}}_{\tilde{\mathcal{G}},nm}\circ H\right] \\  \leq&   4\left(   E_{P_X\times P_Y}[ L^2] \right)^{1/2}        \inf_{\delta>0}\left\{ \delta +1_{\delta<D_\Theta}2\sqrt{2}(nm)^{-1/2} \int_{\frac{1}{2}\delta}^{\frac{1}{2}D_\Theta} \sqrt{\log N(r,\Theta,d_\Theta)} dr\right\}  \notag\,.
     \end{align}
    
    To bound the other term in \eqref{eq:mean_bound_temp}, we start  with the decomposition
  \begin{align}
  \widehat{\mathcal{R}}_{\tilde{\mathcal{G}}_{P_Y},n}(x)
  \leq&\widehat{\mathcal{R}}_{\psi\tilde{\mathcal{G}}_{P_Y},n}(x)+\widehat{\mathcal{R}}_{(1-\psi)\tilde{\mathcal{G}}_{P_Y},n}(x)\\
  \leq&\widehat{\mathcal{R}}_{\psi\tilde{\mathcal{G}}_{P_Y},n}(x)+\frac{1}{n}\sum_{i=1}^n (1-\psi)(x_i)h(x_i)\,.\notag
  \end{align}      
    Hence
  \begin{align}
  \mathcal{R}_{\tilde{\mathcal{G}}_{P_Y},P_X,n}\leq& E_{P_X^n}[1_{\mathcal{X}_*^n}\widehat{\mathcal{R}}_{\psi\tilde{\mathcal{G}}_{P_Y},n}]+ E_{P_X}[(1-\psi)h]\,.
  \end{align} 
  
  For $x\in\mathcal{X}_*^n$, if $x_i\in B$ define $\tilde{x}_i=x_i$, otherwise define $\tilde{x}_i=x^\prime$. Then $\tilde{x}\in(B\cap\mathcal{X}_*)^n$ and
  \begin{align}
  \widehat{\mathcal{R}}_{\psi\tilde{\mathcal{G}}_{P_Y},n}(x)=\widehat{\mathcal{R}}_{\psi\tilde{\mathcal{G}}_{P_Y},n}(\tilde{x})\leq \sup_{\tilde{x}\in(B\cap\mathcal{X}_*)^n}\widehat{\mathcal{R}}_{\psi\tilde{\mathcal{G}}_{P_Y},n}(\tilde{x})\,.
  \end{align}
  Finally, we can bound the empirical Rademacher complexity on $(B\cap\mathcal{X}_*)^n$ by using   Dudley's entropy integral, applied to the function class restricted to $B\cap\mathcal{X}_*$:
  \begin{align}
  &\sup_{\tilde{x}\in(B\cap\mathcal{X}_*)^n}\widehat{\mathcal{R}}_{\psi\tilde{\mathcal{G}}_{P_Y},n}(\tilde{x})\\
  \leq& \inf_{\delta>0}\left\{4\delta +1_{\delta<b-a}8\sqrt{2}n^{-1/2}\int_{\frac{1}{2}\delta}^{\frac{1}{2}(b-a)}\sqrt{\log N(r,(\psi\tilde{\mathcal{G}}_{P_Y})|_{B\cap\mathcal{X}_*} ,\|\cdot\|_\infty)} dr\right\}\,.\notag
  \end{align}
  Combining these completes the proof of \eqref{eq:mean_bound}. When $\Theta$ is  a countable dense subset of the unit ball   in $\mathbb{R}^k$ under some choice of norm then  \eqref{eq:Dudley_parametric_unit_ball}  follows from the fact that  $D_\Theta=2$ together with the standard  covering number bound
  \begin{align}
  N(r,\Theta,d_\Theta)\leq (1+2/r)^k\,.
  \end{align} 
   \end{proof}
  
 In our application to DSM, the space  $\tilde{\mathcal{G}}_{P_Y}$ will consist of   smooth functions.  In this case we can use Lemma \ref{lemma:mean_bound} to obtain the following result.  
  \begin{theorem}\label{thm:mean_bound_smooth}
  Working under Assumption \ref{assump:ULLN_parameterized} where $\mathcal{X}=\mathbb{R}^d$, equipped with a norm $\|\cdot\|$,  and employing the definitions \eqref{eq:tilde_G_defs}, suppose we also have the following:
 \begin{enumerate}
  \item $K,K^\prime,q^\prime>0$, $x_0\in\mathbb{R}^d$ such that 
  \begin{align}\label{eq:tail_bound_main}
  P_X(\|x-x_0\|>t)\leq Ke^{-K^\prime t^{q^\prime}} \,\,\,\,\text{ for all }t>0.
  \end{align}
  \item   $k^\prime\in\mathbb{Z}^+$ such that $\tilde{\mathcal{G}}_{P_Y}\subset C^{k^\prime+1}$ and   $A_j,r_j\geq 0$, $j=0,...,k^\prime+1$ such that
  \begin{align}\label{eq:deriv_bound_assump}
  \|D^j \tilde{g}(x)\|\leq A_j(1+\|x-x_0\|^{r_j})\coloneqq h_j(x)
  \end{align}
  for all $x\in\mathbb{R}^d$, $\tilde{g}\in\tilde{\mathcal{G}}_{P_Y}$ (here, $\|\cdot\|$ denotes the appropriate induced norm for each $j$ and $D^0\tilde{g}\coloneqq \tilde{g}$).
     \item We have a measurable set $\mathcal{X}_*\subset\mathbb{R}^d$ with $P_X(\mathcal{X}_*^c)=0$.
  \item  We have $C_*>0$, $d_*>2(k^\prime+1)$ such that for all $R>0$,
  \begin{align}\label{eq:covering_number_assump_main}
  N(r,\mathcal{X}_*\cap \overline{B}_{R}(x_0),\|\cdot\|)\leq (1+C_*R/r)^{d_*}\,\,\,\text{ for all }r>0\,.
  \end{align}
  \end{enumerate}
   Then,  with $\phi_\pm$ as defined in \eqref{eq:phi_pm_def}, we have
 \begin{align}\label{eq:mean_bound_main}
 \mathbb{E}[\phi_\pm(Z)]
 \leq&     2 \left(   E_{P_X\times P_Y}[ L^2] \right)^{1/2}        \inf_{\delta>0}\left\{4\delta +1_{\delta<D_\Theta}8\sqrt{2}(nm)^{-1/2} \int_{\frac{1}{2}\delta}^{\frac{1}{2}D_\Theta} \sqrt{\log N(r,\Theta,d_\Theta)} dr\right\}\notag\\
 & +  \mathrm{polylog}(n) n^{-\frac{k^\prime+1}{d_*}}  \,,
 \end{align}
    where $D_\Theta$ is the diameter of $\Theta$. Here and in the following, $\mathrm{polylog}(n)$ denotes a polynomial (including a possible constant term) in $\log(n)$ whose coefficients do not  depend on  $n$ or   $m$.
  \end{theorem}
  \begin{remark}\label{remark:unit_ball_scaling}
  The covering number bound \eqref{eq:covering_number_assump_main} corresponds to the intrinsic dimension of the support of $P_X$ being at most $d_*$.   When $\Theta$ is  a countable   subset of a ball in $\mathbb{R}^k$ (in some norm) then we obtain   a bound of the form
  \begin{align}
   \mathbb{E}[\phi_\pm(Z)]\leq \tilde{C}\sqrt{\frac{k}{nm}}+ \mathrm{polylog}(n) n^{-\frac{k^\prime+1}{d_*}}
  \end{align}
  for some $\tilde{C}>0$.  Therefore for large $m$ and up to a $\mathrm{polylog}(n)$ factor (which, in this case, has coefficients that also do not depend on $k$), the bound   scales   according to  $n^{-\frac{k^\prime+1}{d_*}}$, which is controlled by the ratio of smoothness to the intrinsic dimensionality, irrespective of the embedding dimension.  Thus our ULLN mean bound mitigates the curse of dimensionality. 
  \end{remark}
  A detailed proof is given in Appendix \ref{app:mean_bound_smooth}. The key ideas are as follows. Given $R>0$, we let $\psi=\psi_R$ be a smooth bump function supported in the ball $\overline{B}_R(x_0)$.   Based on the tail behavior \eqref{eq:tail_bound_main}, letting  $R=(1+C\log(n))^{1/q^\prime}$ for an appropriate $C$ leads to the first term in the bound \eqref{eq:mean_bound} being $O(n^{-1/2})$.  Applying the techniques from \cite{kolmogorov1959varepsilon} for obtaining covering number bounds for spaces of smooth functions on a bounded set with intrinsic dimension $d_*$  to the space  $\psi_R\tilde{\mathcal{G}}_{P_Y}$,  restricted to $\overline{B}_R(x_0)\cap\mathcal{X}_*$, gives an $O(\mathrm{polylog}(n)n^{-(k^\prime+1)/d_*})$ bound on the last term in \eqref{eq:mean_bound}; in particular, the $R$-dependence of the bounds on  functions in  $\psi_R\tilde{\mathcal{G}}_{P_Y}$,  and their  derivatives, together with the chosen $R$ scaling implies that the $R$-dependence in the covering number bounds only contributes a $\mathrm{polylog}$ factor to the Rademacher complexity bound.

 \section{Application to Score-Based Generative Models}\label{sec:DSM}  
 In this section we use the results developed above to bound the statistical error of   score-based generative models.  This application leverages the key features of our method. It employs auxiliary Gaussian random variables; hence the auxiliary distribution does not have low-dimensional support, even through the data distribution often does. Moreover, when the data distribution has unbounded support, the inherent coupling between the Gaussian samples and the training data samples naturally leads to the type of non-uniform MGF bounds that can be treated via our results but not by previous related concentration inequalities.

   Score-based generative models are trained using a pair of forward (noising) and time-reversed (denoising)  SDEs   \cite{songscore}.  Samples from the target distribution are noised and these noise samples are then used to train a score model (vector field) $s^\theta(x,t)$; this is then  used in the denoising SDE to generate new samples.     There are several different methods that can be used to train the score model; here we will focus on the commonly used   denoising score matching (DSM)  objective  \cite{vincent2011connection}.  Training a score model via DSM leads to a stochastic optimization problem of the form studied in Theorem \ref{thm:stoch_opt}. More specifically, 
   given $c\geq 0$ and a continuous noise strength $\sigma_t>0$, the noising dynamics of a training sample $x\in\mathbb{R}^d$, with corresponding   noise sample $y\in\mathbb{R}^d$,     are defined by 
   \begin{align}
z_{t}\coloneqq e^{-ct}x+\widetilde{\sigma}_{t}y\,,\,\,\,\,\,  \,\,\, \widetilde{\sigma}_t^2\coloneqq\int_0^t e^{-2c(t-s)}\sigma_s^2 ds\,.
   \end{align}
Given   timesteps $t_j\in(0,T]$, $j=1,\dots,J$ and weights $\eta_j\in(0,\infty)$, the DSM training problem is then a stochastic optimization problem of the form \eqref{eq:stoch_opt_emp_prob} with objective
   \begin{align}\label{eq:DSM_g_theta}
   &g_\theta(x,y)\coloneqq \sum_{j=1}^J \eta_j \left\|s^{\theta}( e^{-ct_j}x+\widetilde{\sigma}_{t_j}y,t_j)+\widetilde{\sigma}^{-1}_{t_j}y\right\|^2\,,
   \end{align}
     employing samples $X_i\sim P_{X}$ (the data distribution), $Y_{i,j}\sim N(0,I_d)$ (the noise distribution); in this section  $\|\cdot\|$    denotes the $2$-norm on $\mathbb{R}^d$.  As mentioned above, this problem naturally incorporates the key novel features of the tools developed in this paper:  The auxiliary noising random variables are easy to sample at will, thus it is natural to reuse the training data with different Gaussian samples in the noising dynamics (i.e., over multiple training epochs). The data distribution $P_X$ often has low-dimensional support, but the auxiliary Gaussian distribution does not.   The score model is generally unbounded and the objective \eqref{eq:DSM_g_theta} couples the $x$ and $y$ samples.  When the support of the data distribution is unbounded, this naturally leads to the necessity of considering MGF bounds of the form \eqref{eq:gen_McDiarmid_non_unif_Y_MGF_assump} which are not uniform in   $x_i$.  
     
 \begin{theorem}\label{thm:DSM}
 Suppose we have a data distribution $P_{X}$ on $\mathbb{R}^d$,  a nonempty separable pseudometric space $(\Theta,d_\Theta)$  equipped with a $\sigma$-algebra, and a measurable  score model $s:\Theta\times\mathbb{R}^d\times [0,T]\to \mathbb{R}^d$  that satisfy the following.
 \begin{enumerate}
 \item  We have $\tilde{q}\geq 1$ such that for all $(z,t)$, the map $\theta\mapsto s^\theta(z,t)$ is Lipschitz (with respect to $(d_\Theta,\|\cdot\|)$) with Lipschitz constant $\alpha_t+\beta_t\|z\|^{\tilde{q}}$.
 \item  We have $k^\prime\in\mathbb{Z}^+$ such that, for all $t\in[0,T]$, the map $z\mapsto s^\theta(z,t)$ is $C^{k^\prime+1}$ with bounded first derivative and polynomially-bounded derivatives  of order $2,...,k^\prime+1$, with these bounds all holding  uniformly in  $\theta$. 
 \item $\Theta$ has finite diameter, $D_\Theta$.
 \item The function $x\in\mathbb{R}^d\mapsto\|x\|$ is sub-Gaussian under $P_X$.
    \item We have a measurable set $\mathcal{X}_*\subset\mathbb{R}^d$ with $P_X(\mathcal{X}_*^c)=0$.
      \item  We have $C_*>0$, $d_*>2(k^\prime+1)$ such that for all $R>0$,
           \begin{align} 
           N(r,\mathcal{X}_*\cap \overline{B}_{R}( 0),\|\cdot\|)\leq (1+C_*R/r)^{d_*}\,\,\,\text{ for all }r>0\,.
           \end{align}
 \end{enumerate}
 
 Suppose  we have a probability space $(\Omega,\mathbb{P})$ on which we have  $X_i\sim P_{X}$,   and $Y_{i,j}\sim P_Y\coloneqq N(0,I_d)$, $i=1,\dots,n$, $j=1,\dots,m$,  that are all mutually independent, 
  along with a $\Theta$-valued random variable $\theta^*_{n,m}$ that is an approximate minimizer of the empirical DSM problem:
 \begin{align}\label{eq:dsm_empirical_optim}
 &\frac{1}{nm}\sum_{i=1}^n \sum_{j=1}^m g_{\theta^*_{n,m}}(X_i,Y_{i,j})\leq \inf_{\theta\in\Theta}\frac{1}{nm}\sum_{i=1}^n \sum_{j=1}^m g_{\theta}(X_i,Y_{i,j})+\epsilon_{n,m}^{\text{opt}} \,\,\,\,\,\mathbb{P}\text{-a.s.}\,,
 \end{align}
 where $\epsilon_{n,m}^{\text{opt}}\geq 0$ and $g_\theta$ is as defined in \eqref{eq:DSM_g_theta}.

  Let      $\rho>0$   such that $\rho m>1$.  Then     for all $\epsilon>0$  we have the following statistical error bounds for DSM:
      \begin{align}\label{eq:DSM_conc_ineq}
       &\mathbb{P}\left(E_{P_X\times P_Y}[g_{\theta^*_{n,m}}]- \inf_\theta E_{P_X\times P_Y}[g_\theta]\geq \epsilon+2R_{n,m}^{\mathrm{DSM}}+\epsilon_{n,m}^{\mathrm{opt}}\right)\\
               \leq&  \exp\left(-n\!\sup_{\lambda\in[ 0,\lambda_{m}^{\mathrm{DSM}})}\!\left\{\lambda \epsilon/2-\xi_{m}^{\mathrm{DSM}}(\lambda)\right\}\right)\,,\notag
             \end{align} 
   where $\alpha$ and $\beta$ were defined in \eqref{eq:alpha_beta_def} (for any choice of  $s\in(0,1)$, not to be confused with the score model) and
\begin{align}   
\xi_{m}^{\mathrm{DSM}}(\lambda)\coloneqq &  \frac{2 (1-(\rho m)^{-1})^{-1}(C\lambda)^2}{1-(C(1-(\rho m)^{-1})^{-1}\lambda)^2}         
+ m^{-1}\left(\frac{ 2(\alpha \lambda)^2}{1-\rho (\alpha\lambda)^2}+\frac{2(\beta\lambda )^{2}}{1- ( m^{-1}\beta\lambda)^{2}}\right)\,, \\
\lambda_{m}^{\mathrm{DSM}} \coloneqq & \min\left\{ m\beta^{-1}, \left(\rho \alpha^2\right)^{-1/2} ,(1-(\rho m)^{-1})/C\right\} \,,\notag\\
R_{n,m}^{\mathrm{DSM}}\coloneqq & 2 \left(   E_{P_X\times P_Y}[ L^2] \right)^{1/2}        \inf_{\delta>0}\left\{4\delta +1_{\delta<D_\Theta}8\sqrt{2}(nm)^{-1/2} \int_{\frac{1}{2}\delta}^{\frac{1}{2}D_\Theta} \sqrt{\log N(r,\Theta,d_\Theta)} dr\right\}\notag\\
                   & +  \mathrm{polylog}(n) n^{-\frac{k^\prime+1}{d_*}}\,,\notag
                \end{align}
 with the constants used in these definitions defined and shown to be finite in the course of the proof below.
 \end{theorem}
 \begin{remark} 
   Per the discussion in Remark \ref{remark:unit_ball_scaling},  for appropriate finite dimensional parameterizations $\Theta\subset\mathbb{R}^k$,   we have $R_{n,m}^{\mathrm{DSM}}=O(\sqrt{k/(nm)}+ \mathrm{polylog}(n) n^{-\frac{k^\prime+1}{d_*}})$ as $n\to\infty$, where the implied constant does not depend on $n,m,k$.  As previously discussed in  Remark  \ref{remark:ULLNlarge_m_limit}, the bound \eqref{eq:DSM_conc_ineq}  also   captures the benefits of sample reuse in that for large $m$,  the leading order behavior of the bound   depends on the $x$-difference bounds, with the $Y$-dependence    integrated out.
  The supremum on the right-hand side of \eqref{eq:DSM_conc_ineq}  is reminiscent of those appearing in Bernstein-type concentration inequalities, apart from the denominators involving $\lambda^2$ in place of $\lambda$.  However, the asymptotic behavior for small $\epsilon$ is the same, leading to an upper bound of the form $\exp(-n\tilde{A}\epsilon^2)$ for some $\tilde{A}\in(0,\infty)$.   Hence, assuming zero optimization error, for any $\delta\in(0,1)$ our result implies an overall statistical error bound of the form
   \begin{align}
   E_{P_X\times P_Y}[g_{\theta^*_{n,m}}]- \inf_\theta E_{P_X\times P_Y}[g_\theta]=O\left(  \sqrt{k/(nm)}+ \mathrm{polylog}(n) n^{-\frac{k^\prime+1}{d_*}}\right),
   \end{align}
   with probability at least $1-\delta$.   Therefore, by taking $m$ large, the   term that depends on the dimension of the parameter space can be essentially neglected and the scaling of the error is determined by the intrinsic dimension of the support of the data distribution, $d_*$, compared to the degree of smoothness of the score model, $k^\prime+1$; thus our bound addresses the curse of dimensionality, despite the fact that the support of the $Y$-distribution is all of $\mathbb{R}^d$.  
       \end{remark}  
 \begin{proof}
 For $t\in[0,T]$, boundedness of the first derivative (uniformly in $\theta$), the assumed Lipschitz property of the score in $\theta$, along with $D_\Theta<\infty$   implies that   we have $a_t,b_t,L_t\in[0,\infty)$ such that: 1) For all  $\theta\in\Theta$, $z\mapsto s^\theta(z,t)$ is $L_t$-Lipschitz,  2)      $\|s^\theta(z,t)\|\leq a_t+b_t\|z\|$  for all $z\in\mathbb{R}^d$, $\theta\in\Theta$.

 Now we will show that the assumptions of Theorem \ref{thm:stoch_opt} hold.  First we rewrite
 \begin{align}
 g_\theta(x,y)
 =& \sum_{j=1}^J \eta_j \left(\|s^{\theta}( e^{-ct_j}x+\widetilde{\sigma}_{t_j}y,t_j)\|^2+2\widetilde{\sigma}^{-1}_{t_j} s^{\theta}( e^{-ct_j}x+\widetilde{\sigma}_{t_j}y,t_j)\cdot y+\widetilde{\sigma}^{-2}_{t_j}\|y\|^2\right)\,.
 \end{align}
  From this we can compute the growth bound
   \begin{align}
  &|g_\theta(x,y)|\\
  \leq&\sum_{j=1}^J \eta_j \!\left((a_{t_j}+b_{t_j}\| e^{-ct_j}x+\widetilde{\sigma}_{t_j}y\|)^2+2\widetilde{\sigma}^{-1}_{t_j}(a_{t_j}+b_{t_j}\| e^{-ct_j}x+\widetilde{\sigma}_{t_j}y\|)\|y\|+\widetilde{\sigma}^{-2}_{t_j}\|y\|^2\right)\notag\\
  &\coloneqq h(x,y)\in L^1(P_{X}\times P_Y)\,,\notag
  \end{align}
   and the Lipschitz bound  
  \begin{align}
  &|g_{\theta_1}(x,y)-g_{\theta_2}(x,y)|	\\
  \leq&  \sum_{j=1}^J \eta_j\left| \|s^{\theta_1}( e^{-ct_j}x+\widetilde{\sigma}_{t_j}y,t_j)\|^2- \|s^{\theta_2}( e^{-ct_j}x+\widetilde{\sigma}_{t_j}y,t_j)\|^2\right|\notag\\
  & + 2\sum_{j=1}^J \eta_j\widetilde{\sigma}^{-1}_{t_j}\| s^{\theta_1}( e^{-ct_j}x+\widetilde{\sigma}_{t_j}y,t_j) - s^{\theta_2}( e^{-ct_j}x+\widetilde{\sigma}_{t_j}y,t_j)\|\|y\|\notag\\
   \leq&  2\sum_{j=1}^J \eta_j(a_{t_j}+b_{t_j}\|e^{-ct_j}x+\widetilde{\sigma}_{t_j}y\|)(\alpha_{t_j}+\beta_{t_j}\| e^{-ct_j}x+\widetilde{\sigma}_{t_j}y\|^{\tilde{q}})d_\Theta(\theta_1,\theta_2)\notag\\
   & + 2\sum_{j=1}^J \eta_j\widetilde{\sigma}^{-1}_{t_j} (\alpha_{t_j}+\beta_{t_j}\| e^{-ct_j}x+\widetilde{\sigma}_{t_j}y\|^{\tilde{q}})\|y\|d_\Theta(\theta_1,\theta_2)  \,.\notag
   \end{align}  
   Hence $\theta\mapsto g_\theta(x,y)$ is $L(x,y)$-Lipschitz, where 
     \begin{align}\label{eq:L_DSM_def}
     L(x,y)\coloneqq&  2\sum_{j=1}^J \eta_j(a_{t_j}+b_{t_j}\|e^{-ct_j}x+\widetilde{\sigma}_{t_j}y\|)(\alpha_{t_j}+\beta_{t_j}\| e^{-ct_j}x+\widetilde{\sigma}_{t_j}y\|^{\tilde{q}}) \\
       & + 2\sum_{j=1}^J \eta_j\widetilde{\sigma}^{-1}_{t_j} (\alpha_{t_j}+\beta_{t_j}\| e^{-ct_j}x+\widetilde{\sigma}_{t_j}y\|^{\tilde{q}})\|y\|\notag\,.
    \end{align}
  Therefore we have proven that Assumption \ref{assump:stoch_opt_nonuniform_bound} holds. We also note that  $L\in L^2(P_X\times P_Y)$ due to $P_Y$ being Gaussian and $\|x\|$ being sub-Gaussian under $P_X$.

  Next we verify that  Assumption \ref{assump:ULLN} holds for $\mathcal{G}\coloneqq\{g_\theta:\theta\in\Theta\}$. First we compute
  \begin{align}
  &|g_\theta(x,y)-g_\theta({x},\tilde{y})|\\
  \leq&\sum_{j=1}^J \eta_j\|s^{\theta}( e^{-ct_j}x+\widetilde{\sigma}_{t_j}y,t_j)-s^{\theta}( e^{-ct_j}x+\widetilde{\sigma}_{t_j}\tilde{y},t_j)+\widetilde{\sigma}^{-1}_{t_j}(y-\tilde{y})\|\notag\\
  &\times\left(\|s^{\theta}( e^{-ct_j}x+\widetilde{\sigma}_{t_j}y,t_j)+\widetilde{\sigma}^{-1}_{t_j}y\|+\|s^{\theta}( e^{-ct_j}x+\widetilde{\sigma}_{t_j}\tilde{y},t_j)+\widetilde{\sigma}^{-1}_{t_j}\tilde{y}\|\right)\notag\\
   \leq&\sum_{j=1}^J \eta_j\left(\widetilde{\sigma}_{t_j} L_{t_j}
   +\widetilde{\sigma}^{-1}_{t_j}\right)\left(2a_{t_j}+2b_{t_j}e^{-ct_j}\|x\|+(b_{t_j}\widetilde{\sigma}_{t_j} +\widetilde{\sigma}^{-1}_{t_j})(\|y\|+\|\tilde{y}\|)\right)\| y- \tilde{y}\| \notag\\
     =&h_{0,\mathcal{Y}}(y,\tilde{y}) +r_{\mathcal{X}}(x)k_{\mathcal{Y}}(y,\tilde{y})\notag\,,
     \end{align}
     where
     \begin{align}\label{eq:r_h0_k_def}
     r_{\mathcal{X}}(x)\coloneqq &\sum_{j=1}^J \eta_j\left(\widetilde{\sigma}_{t_j} L_{t_j}
     +\widetilde{\sigma}^{-1}_{t_j}\right)2(a_{t_j}+b_{t_j}e^{-ct_j}\|x\|)\,,\,\,\,\,\,k_{\mathcal{Y}}(y,\tilde{y})\coloneqq \| y- \tilde{y}\|\,,\\
     h_{0,\mathcal{Y}}(y,\tilde{y})    \coloneqq &\sum_{j=1}^J \eta_j\left(\widetilde{\sigma}_{t_j} L_{t_j}
     +\widetilde{\sigma}^{-1}_{t_j}\right)(b_{t_j}\widetilde{\sigma}_{t_j} +\widetilde{\sigma}^{-1}_{t_j})(\|y\|+\|\tilde{y}\|)\| y- \tilde{y}\|\notag\,.
      \end{align}
      Note that $\|y\|$ is sub-Gaussian under  $P_Y=N(0,I_d)$, hence $\|y\|+\|\tilde{y}\|$ and $\|y-\tilde{y}\|$ are sub-Gaussian under $P_Y^2$. Therefore $\|k_{\mathcal{Y}}\|_{\Phi_2,P_Y^2}<\infty$. This also implies that $h_{0,\mathcal{Y}}$ is sub-exponential  under $P_Y^2$, and hence   $\|h_{0,\mathcal{Y}}\|_{\Phi_1,P_Y^2}<\infty$.  Finally, we have assumed that $\|x\|$ is sub-Gaussian under $P_X$, hence $\|r_{\mathcal{X}}\|_{\Phi_2,P_X}<\infty$.

 Next, we compute
 \begin{align}
 &|g_\theta(x,y)-g_\theta(\tilde{x},y)|\\
 \leq& \sum_{j=1}^J \eta_j\|s^{\theta}( e^{-ct_j}x+\widetilde{\sigma}_{t_j}y,t_j)-s^{\theta}( e^{-ct_j}\tilde{x}+\widetilde{\sigma}_{t_j}y,t_j)\|\notag\\
 &\times\left(\|s^{\theta}( e^{-ct_j}x+\widetilde{\sigma}_{t_j}y,t_j)+\widetilde{\sigma}^{-1}_{t_j}y\|+\|s^{\theta}( e^{-ct_j}\tilde{x}+\widetilde{\sigma}_{t_j}y,t_j)+\widetilde{\sigma}^{-1}_{t_j}y\|\right)\notag\\
 \leq& \sum_{j=1}^J \eta_j  L_{t_j}e^{-ct_j}\left( 2a_{t_j}+2(b_{t_j}\widetilde{\sigma}_{t_j}+ \widetilde{\sigma}^{-1}_{t_j})\|y\|+b_{t_j}e^{-ct_j}(\|x\|+\|\tilde{x}\|)\right)\|  x- \tilde{x}\|\notag\\
 \coloneqq& h_{\mathcal{X}}(x,\tilde{x},y)\,.\notag
 \end{align}
 We have $h_{\mathcal{X}}\in L^1(P_X\times P_X\times P_Y)$ and  
 \begin{align}\label{eq:DSM_h_X_MGF_bound}
 &\int \cosh(\lambda E_{P_Y}[h_{\mathcal{X}}(x,\tilde{x},\cdot)])P_X^2(dxd\tilde{x})\\
 =&\int \cosh\big(\lambda\left(
 A\|  x- \tilde{x}\|+B(\|x\|+\|\tilde{x}\|)\|  x- \tilde{x}\|\right)
  \big)P_X^2(dxd\tilde{x})\notag\\
  \leq& \exp\left( \frac{2(C\lambda)^2}{1-(C\lambda)^2} \right) \notag
  \end{align}
  for all $\lambda\in[0,1/C)$,
  where 
   \begin{align}\label{eq:DSM_A_B_def}
     A\coloneqq & \sum_{j=1}^J \eta_j  L_{t_j}e^{-ct_j} \left(2a_{t_j}+2(b_{t_j}\widetilde{\sigma}_{t_j}+ \widetilde{\sigma}^{-1}_{t_j})E_{P_Y}[\|y\|]\right)\,,\,\,\,\,
     B\coloneqq \sum_{j=1}^J \eta_j  L_{t_j}b_{t_j}e^{-2ct_j} \,,\\
     C\coloneqq &\Big\|A\|  x- \tilde{x}\|+B(\|x\|+\|\tilde{x}\|)\|  x- \tilde{x}\|\Big\|_{\Phi_1,P_X^2(dxd\tilde{x})}\,.\notag
      \end{align}
      Note that the sub-Gaussian assumption on $\|x\|$ implies that $C<\infty$.     
  To obtain the bound \eqref{eq:DSM_h_X_MGF_bound},  we used the $k=1$ case of Lemma \ref{lemma:Orlicz_cosh_MGF_bound_sum}, which was previously derived in Corollary 2.4 of \cite{doi:10.1137/25M1748664}.
  
 We have verified that Assumption \ref{assump:ULLN}  holds with  $ \xi_{\mathcal{X}}(\lambda)\coloneqq 
   \frac{2(C\lambda)^2}{1-(C\lambda)^2}$ and  $K_{\mathcal{X}}\coloneqq 1/C$.  Therefore Theorem \ref{thm:stoch_opt} implies
   \begin{align}\label{eq:conc_ineq_DSM_temp}
   &\mathbb{P}\left(E_{P_X\times P_Y}[g_{\theta^*_{n,m}}]- \inf_\theta E_{P_X\times P_Y}[g_\theta]\geq \epsilon+2\mathbb{E}\left[\phi_{1/2}(Z)\right]+\epsilon_{n,m}^{\mathrm{opt}}\right)\\
           \leq&  \exp\left(-n\!\sup_{\lambda\in[ 0,\lambda_{m}^{\mathrm{DSM}})}\!\left\{\lambda \epsilon/2-\xi_{m}^{\mathrm{DSM}}(\lambda)\right\}\right)\,,\notag
         \end{align} 
     where   $\alpha$ and $\beta$ were defined in \eqref{eq:alpha_beta_def}, and $\phi_{1/2}$ is as defined in Theorem \ref{thm:ULLN_non_uniform_mean_diff_bounds}.

     Next we bound the mean of $\phi_{1/2}(Z)$. We have $\phi_{1/2}=(\phi_++\phi_-)/2$,     
     where $\phi_\pm$ is given by a supremum over a countable dense subset $\Theta_0\subset\Theta$, as in \eqref{eq:phi_pm_Theta0}.  It is straightforward to check that the assumptions of Theorem \ref{thm:mean_bound_smooth} hold; in particular, we note that the tail bound \eqref{eq:tail_bound_main} follows from the assumption that $x\mapsto\|x\|$ is sub-Gaussian and   derivative bounds of the form \eqref{eq:deriv_bound_assump} follow from our assumed properties of the score model. Therefore we can conclude
      \begin{align} 
      \mathbb{E}[\phi_\pm(Z)]
      \leq&     2 \left(   E_{P_X\times P_Y}[ L^2] \right)^{1/2}        \inf_{\delta>0}\left\{4\delta +1_{\delta<D_\Theta}8\sqrt{2}(nm)^{-1/2} \int_{\frac{1}{2}\delta}^{\frac{1}{2}D_\Theta} \sqrt{\log N(r,\Theta,d_\Theta)} dr\right\}\notag\\
      & +  \mathrm{polylog}(n) n^{-\frac{k^\prime+1}{d_*}}  \,,
      \end{align}
         where $D_\Theta$ is the diameter of $\Theta$ and $\mathrm{polylog}(n)$ denotes a polynomial in $\log(n)$ whose coefficients do not  depend on  $m$.    Combining this with \eqref{eq:conc_ineq_DSM_temp}  completes the proof. 
  \end{proof}

 As shown in   \cite{hyvarinen2005estimation,vincent2011connection}, the difference between the   score model $s^\theta(x,t)$  and the exact score $s(x,t)$ can be related to the approximation error and the DSM objective studied above via
  \begin{align}
&\sum_{j=1}^J \eta_j{ E}_{Z_{t_j}}\left[ \|s^{\theta}(\cdot,t_j)-s(\cdot,t_j)\|^2\right]-  \inf_{\theta\in\Theta} \sum_{j=1}^J \eta_j{ E}_{ Z_{t_j}}\!\!\left[ \|s^\theta(\cdot,t_j)-s(\cdot,t_j)\|^2\right]\\
=& E_{P_{X}\times P_Y}\left[g_\theta\right]
  -\inf_{\theta\in\Theta}{E}_{P_{X}\times P_Y}\left[g_\theta\right]\,.\notag
  \end{align}
 The  statistical  error  bounds on the DSM objective from Theorem \ref{thm:DSM} therefore directly translate into statistical performance guarantees, in the form of excess  error bounds on the score model.  Moreover, such bounds also feed into the assumptions made by convergence analyses such as those in \cite{chen2023improved,bentonnearly,li2024towards,li2024sharp,potaptchik2024linear,chen2024equivariant,mimikos2024score}.  We do acknowledge that the   smoothness and growth assumptions  in Theorem \ref{thm:DSM}   do not apply to all  implementations used in practice, though  they can be enforced by  architectural choices, including spectral normalization \cite{miyato2018spectral}.    Finally, emphasize that there are other important contributions to the error in score-based generative modeling; see  \cite{oko2023diffusion,chakraborty2026generalization} for more  in-depth analyses of multiple aspects of the DSM problem.  
  Our purpose with this example is not a complete analysis of DSM; rather, it as an important example which demonstrates the utility and relative ease of use of our our multi-purpose tools.

  \section{Conclusion}
    We have derived a novel generalization of McDiarmid's inequality, Theorem \ref{thm:gen_McDiarmid_non_unif_MGF}, which    allows for a more general form of  unbounded one-component differences than were covered by previous  extensions of McDiarmid's inequality.   To facilitate explicit computations, we also developed tools for computing the required non-uniform MGF bounds via Orlicz norms; see Section \ref{sec:Orlicz} and Theorem \ref{thm:ULLN_bound_non_unif_MGF_Phi1_Phi2_case}. Using these results, in Section \ref{sec:ULLN_stoch_opt} we derived a novel ULLN result as well as statistical error bounds on stochastic optimization problems. Our results   capture  the benefits of   training sample reuse in algorithms that employ easily sampled auxiliary random variables, in addition to the training data. By   using the accompanying ULLN mean bounds developed in Section \ref{sec:mean_bound}, our method also addresses the curse of dimensionality when the data distribution has low-dimensional support.  In Section \ref{sec:DSM} we applied these tools to derive statistical error bounds on score-based  generative modeling via denoising score matching, an application which naturally leverages the key features of our results. As a consequence,  our statistical performance guarantees on DSM  capture the benefit of training sample reuse, apply to     unbounded objectives  and data distributions with unbounded support, and the leading behavior is governed by the ratio of the intrinsic dimension of the support of the data distribution to the degree of  smoothness of the score model, thus addressing the curse of dimensionality.

\appendix

\section{ULLN for Non-Identically Distributed Random Variables}
In this appendix we derive  a ULLN bound that applies when the factors of the product measure are non-identical (we emphasize that we still assume independence).  The proof closely mirrors the standard result (i.e., all factors identical) from, e.g., Theorem 3.3 in \cite{mohri2018foundations}.   
\begin{lemma}\label{lemma:ULLN_non-identical}
Let $\mathcal{Z}$ be a measurable space and $\mathcal{G}$ be a nonempty countable family of real-valued  measurable functions on $\mathcal{Z}$.  Let $\mu_i$, $i=1,...,m$ be   probability measures on $\mathcal{Z}$ with $\mathcal{G}\subset L^1(\mu_i)$ for all $i$.  Define $\mu^m\coloneqq\prod_{i=1}^m\mu_i$.  Then
\begin{align}
&E_{z\sim\mu^m}\!\left[\sup_{g\in\mathcal{G}}\left\{\pm\left(\frac{1}{m}\sum_{i=1}^m \left(g(z_i)-E_{\mu_i}[g]\right)\right)\right\}\right]
\leq    2E_{\mu^m}\!\!\left[ \widehat{\mathcal{R}}_{\mathcal{G},m}\right]\,,
\end{align}
where
\begin{align}\label{eq:empirical_rademacher_def}
\widehat{\mathcal{R}}_{\mathcal{G},m}(z)\coloneqq E_\sigma\!\left[\sup_{g\in \mathcal{G}}\left\{ \frac{1}{m}\sum_{i=1}^m\sigma_i  g(z_i)\right\}\right]\,,
\end{align}
with $\sigma$ being uniformly distributed on $\{-1,1\}^m$.
\end{lemma}
\begin{remark}\label{remark:Rademacher_def}
We emphasize that
$\widehat{\mathcal{R}}_{\mathcal{G},m}$ is the standard empirical Rademacher complexity; the difference here is that $\mu^m$ is not necessarily the product of identical measures.  When all the measures are identical, $\mu^m=\prod_{i=1}^m \mu$, then we denote the corresponding mean Rademacher complexity by $\mathcal{R}_{\mathcal{G},\mu,m}\coloneqq E_{\mu^m}[\widehat{\mathcal{R}}_{\mathcal{G},m}]$.  However, in the non-identical case we do not use such notation in order to emphasize the distinction from the standard mean Rademacher complexity.
\end{remark}
 \begin{proof}
  Using the standard  symmetrization technique we can write
 \begin{align}
 &E_{z\sim\mu^m}\!\left[\sup_{g\in\mathcal{G}}\left\{\pm\left(\frac{1}{m}\sum_{i=1}^m \left(g(z_i)-E_{\mu_i}[g]\right)\right)\right\}\right]\\
 =&\int \sup_{g\in\mathcal{G}}\left\{ \int \pm\frac{1}{m}\sum_{i=1}^m(g(z_i)-g(z^\prime_i)) \mu^m(dz^\prime)\right\} \mu^m(dz)\notag\\
  \leq &\int \! \int \sup_{g\in\mathcal{G}}\left\{\pm\frac{1}{m}\sum_{i=1}^m(g(z_i)-g(z^\prime_i)) \right\} \mu^m(dz^\prime)\mu^m(dz)\notag\\
  =&\int  \sup_{g\in\mathcal{G}}\left\{\pm\frac{1}{m}\sum_{i=1}^m(g(z_i)-g(z_{-i})) \right\} \prod_{i\in I}\mu_i(dz_i)\,,\notag
 \end{align}
where  $I\coloneqq\{1,...,m,-1,...,-m\}$ and $\mu_{-i}\coloneqq \mu_i$.

 Let $\sigma\in\{-1,1\}^m$ and define    $\Psi_\sigma:\prod_{i\in I}\mathcal{Z}\to \prod_{i\in I}\mathcal{Z}$ by $\Psi_\sigma[z]_i=z_{\sigma_{|i|}\cdot i}$. Given measurable sets $C_i\subset\mathcal{Z}$, we have  
 \begin{align}
\left[(\Psi_\sigma)_\# \prod_{i\in I}\mu_i\right]\left(\prod_{i\in I}C_i\right)  =&  \prod_{i\in I}\mu_i\left(C_{\sigma_{|i|}\cdot i}\right)\\
=&   \prod_{\substack{i\in I:\\
\sigma_{|i|}=1}}\mu_i\left(C_{ i}\right)\times \prod_{\substack{i\in I:\\
\sigma_{|i|}=-1}}\mu_i\left(C_{- i}\right)\notag\\
=&   \prod_{\substack{i\in I:\\
\sigma_{|i|}=1}}\mu_i\left(C_{ i}\right)\times \prod_{\substack{i\in I:\\
\sigma_{|i|}=-1}}\mu_{-i}\left(C_{i}\right)\notag\\
=&   \prod_{\substack{i\in I:\\
\sigma_{|i|}=1}}\mu_i\left(C_{ i}\right)\times \prod_{\substack{i\in I:\\
\sigma_{|i|}=-1}}\mu_{i}\left(C_{i}\right)
=  \left[\prod_{i\in I}\mu_i\right]\left(\prod_{i\in I}C_i\right)\,,\notag
\end{align}
where we used that negation maps $\{i\in I:\sigma_{|i|}=-1\}$ to itself in a bijective manner.  Therefore $(\Psi_\sigma)_\# \prod_{i\in I}\mu_i=\prod_{i\in I}\mu_i$.  Using this we can compute
\begin{align}
&E_{z\sim\mu^m}\!\left[\sup_{g\in\mathcal{G}}\left\{\pm\left(\frac{1}{m}\sum_{i=1}^m \left(g(z_i)-E_{\mu_i}[g]\right)\right)\right\}\right]\\
\leq  &\int  \sup_{g\in\mathcal{G}}\left\{\pm\frac{1}{m}\sum_{i=1}^m(g(\Psi_\sigma[z]_i)-g(\Psi_\sigma[z]_{-i})) \right\} \prod_{i\in I}\mu_i(dz_i)\notag\\
=  &\int  \sup_{g\in\mathcal{G}}\left\{\pm\frac{1}{m}\sum_{i=1}^m\sigma_i(g(z_i)-g(z_{-i})) \right\} \prod_{i\in I}\mu_i(dz_i)\notag\\
=  &\int  \sup_{g\in\mathcal{G}}\left\{\pm\frac{1}{m}\sum_{i=1}^m\sigma_ig(z_i)\mp\frac{1}{m}\sum_{i=1}^m\sigma_ig(z_{-i}) \right\} \prod_{i\in I}\mu_i(dz_i)\notag\\
\leq&\int  \sup_{g\in\mathcal{G}}\left\{\pm\frac{1}{m}\sum_{i=1}^m\sigma_ig(z_i)\right\}\prod_{i\in I}\mu_i(dz_i)+\int\sup_{g\in\mathcal{G}}\left\{\mp\frac{1}{m}\sum_{i=1}^m\sigma_ig(z_{-i}) \right\} \prod_{i\in I}\mu_i(dz_i)\,.\notag
\end{align}

Averaging over $\sigma$ we obtain
\begin{align}
&E_{z\sim\mu^m}\!\left[\sup_{g\in\mathcal{G}}\left\{\pm\left(\frac{1}{m}\sum_{i=1}^m \left(g(z_i)-E_{\mu_i}[g]\right)\right)\right\}\right]\\
\leq  &\int E_\sigma\left[ \sup_{g\in\mathcal{G}}\left\{\pm\frac{1}{m}\sum_{i=1}^m\sigma_ig(z_i)\right\}\right]\prod_{i\in I}\mu_i(dz_i)+\int E_\sigma\left[\sup_{g\in\mathcal{G}}\left\{\mp\frac{1}{m}\sum_{i=1}^m\sigma_ig(z_{-i}) \right\} \right]\prod_{i\in I}\mu_i(dz_i)\notag\\
=  &\int \widehat{\mathcal{R}}_{\pm\mathcal{G},m}(z)\mu^m(dz)+\int \widehat{\mathcal{R}}_{\mp\mathcal{G},m}(z)\mu^m(dz)\notag\\
=  &2\int \widehat{\mathcal{R}}_{\mathcal{G},m}(z)\mu^m(dz) \,.\notag
\end{align}   
 \end{proof}

\section{Proof of the ULLN Mean Bound for Smooth Functions}\label{app:mean_bound_smooth}
In this appendix, we provide the proof of Theorem \ref{thm:mean_bound_smooth}, which is restated below.
 \begin{theorem} 
  Working under Assumption \ref{assump:ULLN_parameterized} where $\mathcal{X}=\mathbb{R}^d$, equipped with a norm $\|\cdot\|$,  and employing the definitions \eqref{eq:tilde_G_defs}, suppose we also have the following:
 \begin{enumerate}
  \item $K,K^\prime,q^\prime>0$, $x_0\in\mathbb{R}^d$ such that 
  \begin{align}\label{eq:tail_bound}
  P_X(\|x-x_0\|>t)\leq Ke^{-K^\prime t^{q^\prime}} \,\,\,\,\text{ for all }t>0.
  \end{align}
  \item   $k^\prime\in\mathbb{Z}^+$ such that $\tilde{\mathcal{G}}_{P_Y}\subset C^{k^\prime+1}$ and   $A_j,r_j\geq 0$, $j=0,...,k^\prime+1$ such that
  \begin{align}
  \|D^j \tilde{g}(x)\|\leq A_j(1+\|x-x_0\|^{r_j})\coloneqq h_j(x)
  \end{align}
  for all $x\in\mathbb{R}^d$, $\tilde{g}\in\tilde{\mathcal{G}}_{P_Y}$ (here, $\|\cdot\|$ denotes the appropriate induced norm for each $j$ and $D^0\tilde{g}\coloneqq \tilde{g}$).
     \item A measurable set $\mathcal{X}_*\subset\mathbb{R}^d$ with $P_X(\mathcal{X}_*^c)=0$
  \item  We have $C_*>0$, $d_*>2(k^\prime+1)$ such that for all $R>0$,
  \begin{align}\label{eq:covering_number_assump}
  N(r,\mathcal{X}_*\cap \overline{B}_{R}(x_0),\|\cdot\|)\leq (1+C_*R/r)^{d_*}\,\,\,\text{ for all }r>0\,.
  \end{align}
  \end{enumerate}
   Then,  with $\phi_\pm$ as defined in \eqref{eq:phi_pm_def}, we have
 \begin{align} 
 \mathbb{E}[\phi_\pm(Z)]
 \leq&     2 \left(   E_{P_X\times P_Y}[ L^2] \right)^{1/2}        \inf_{\delta>0}\left\{4\delta +1_{\delta<D_\Theta}8\sqrt{2}(nm)^{-1/2} \int_{\frac{1}{2}\delta}^{\frac{1}{2}D_\Theta} \sqrt{\log N(r,\Theta,d_\Theta)} dr\right\}\notag\\
 & +  \mathrm{polylog}(n) n^{-\frac{k^\prime+1}{d_*}}  \,,
 \end{align}
    where $D_\Theta$ is the diameter of $\Theta$ and $\mathrm{polylog}(n)$ denotes a polynomial in $\log(n)$ whose coefficients do not  depend on  $m$.
  \end{theorem}
\begin{proof}
  Let $\kappa:\mathbb{R}^d\to[0,1]$ be a $C^\infty$ bump function equal to $1$ on $\overline{B}_{1/2}(0)$   and with support in $\overline{B}_1(0)$.  For $R>0$ define $\psi_R:\mathbb{R}^d\to[0,1]$, $\psi_R(x)=\kappa((x-x_0)/R)$, and $B_R\coloneqq \overline{B}_{R}(x_0)$.  Choose $x^\prime$ with $\|x^\prime-x_0\|=R$.  
  The assumptions of Lemma \ref{lemma:mean_bound} hold with $\psi=\psi_R$, $B=B_R$, $h=h_0$, $b= A_0(1+R^{r_0})\coloneqq \tilde{h}_0(R)$, $a=-\tilde{h}_0(R)$, and with $\mathcal{X}_*$ replaced by $\mathcal{X}_*\cup\{x^\prime\}\coloneqq  \mathcal{X}_*^\prime$. Hence  we obtain
     \begin{align}\label{eq:mean_bound_smooth_start}
        \mathbb{E}[\phi_\pm(Z)]\leq  &2E_{P_X}[(1-\psi_R)h_0]\\
        &+ 2 \left(   E_{P_X\times P_Y}[ L^2] \right)^{1/2}        \inf_{\delta>0}\left\{4\delta +1_{\delta<D_\Theta}8\sqrt{2}(nm)^{-1/2} \int_{\frac{1}{2}\delta}^{\frac{1}{2}D_\Theta} \sqrt{\log N(r,\Theta,d_\Theta)} dr\right\}\notag\\
        & +2 \inf_{\delta>0}\left\{4\delta +1_{\delta<2\tilde{h}_0(R)}8\sqrt{2}n^{-1/2}\int_{\frac{1}{2}\delta}^{\tilde{h}_0(R)}\sqrt{\log N(r,(\psi_R\tilde{\mathcal{G}}_{P_Y})|_{B_R\cap\mathcal{X}_*^\prime} ,\|\cdot\|_\infty)} dr\right\}\,.\notag
        \end{align} 
  
Letting   $p,q\in (1,\infty)$   be   conjugate exponents, we can compute
  \begin{align}
  E_{P_X}[(1-\psi_R)h_0]\leq&E_{P_X}[(1-\psi_R)^q]^{1/q}\|h_0\|_{L^p(P_X)}\\
  \leq&{P_X}(\|x-x_0\|>R/2)^{1/q}\|h_0\|_{L^p(P_X)}\notag\\
   \leq&K^{1/q}e^{-q^{-1}2^{-q^\prime}K^\prime R^{q^\prime} } \|h_0\|_{L^p(P_X)}\,,\notag
  \end{align}
where we note that the tail decay of $P_X$ implies $h_0\in L^p(P_X)$.  Letting 
\begin{align}\label{eq:Rn_def}
R_n=\left(1+\frac{2^{q^\prime-1}q}{K^\prime}\log(n)\right)^{1/q^\prime}
\end{align}
 we therefore  obtain 
  \begin{align}\label{eq:mean_bound_smooth_envelope_term}
  E_{P_X}[(1-\psi_{R_n})h_0]=O(n^{-1/2})\,.
  \end{align}

  Now we proceed to bound the last term in \eqref{eq:mean_bound_smooth_start}.
  By the product rule and induction, it is straightforward to show that there exists polynomials $p_j(R)$, $j=0,...,k^\prime+1$  with
 \begin{align}
 \|D^j(\psi_R\tilde{g})\|\leq p_j(R)
 \end{align}
 for all $R\geq 1$, $\tilde{g}\in\tilde{\mathcal{G}}_{P_Y}$.  Therefore $  \psi_R\tilde{\mathcal{G}}_{P_Y}\subset \mathcal{G}^{k^\prime}_R$, where
  \begin{align}
   \mathcal{G}^{k^\prime}_R\coloneqq \{f\in C^{k^\prime}(\mathbb{R}^d):\|D^j f\|\leq p_j(R)\,, j=0,...,k^\prime,\,  D^{k^\prime}f \text{ is $p_{k^\prime+1}(R)$-Lipschitz}\}\,.
   \end{align}
 The calculations in \cite{kolmogorov1959varepsilon} imply a   covering number bound of the following form for the space of smooth functions $\mathcal{G}^{k^\prime}_R$ when restricted to $B_R\cap\mathcal{X}_*^\prime$:
  \begin{align}
  &\log N\left(r,\mathcal{G}^{k^\prime}_R|_{B_R\cap\mathcal{X}_*^\prime},\|\cdot\|_\infty\right)\\
  \leq&     \left(A+  B\log(R)+\sum_\ell a_\ell \log(1+r^{-b_\ell})\right)  \left(1+ C \frac{R(p_{k^\prime+1}(R))^{ 1/(k^\prime+1)}}{r^{1/(k^\prime+1)}}\right)^{d_*}  \,,\notag
  \end{align}
  for some finite list of constants $A,B,a_\ell,b_\ell,C>0$, where we have used the 
covering number bound 
   \begin{align}
   N(r,B_R\cap\mathcal{X}_*^\prime,\|\cdot\|)
   \leq&\begin{cases}
   1 &\text{ if }r\geq 2R\\
   1+(1+C_*R/r)^{d_*} &\text{ otherwise}
   \end{cases}
   \leq(1+\tilde{C}R/r)^{d_*}
   \end{align}
   for some $\tilde{C}>0$. 
   Therefore, letting $\delta_n=n^{-(k^\prime+1)/d_*}$, upper bounding $r^{-b_\ell}$ with $(\delta_n/2)^{-b_\ell}$ in the integral, and with $R_n$ as in \eqref{eq:Rn_def}, we can compute
  \begin{align}\label{eq:mean_bound_smooth_covering_term}
 &  2 \inf_{\delta>0}\left\{4\delta +1_{\delta<2\tilde{h}_0(R_n)}8\sqrt{2}n^{-1/2}\int_{\frac{1}{2}\delta}^{\tilde{h}_0(R_n)}\sqrt{\log N(r,(\psi_{R_n}\tilde{\mathcal{G}}_{P_Y})|_{B_{R_n}\cap\mathcal{X}_*^\prime} ,\|\cdot\|_\infty)} dr\right\}\\
 \leq&O(n^{-(k^\prime+1)/d_*})+  n^{-1/2} \mathrm{polylog}(n)  \left(\tilde{h}_0(R_n)+\tilde{C}R_n^{d_*/2}(p_{k^\prime+1}(R_n))^{ \frac{d_*}{2(k^\prime+1)}} \int_{\frac{1}{2}\delta_n}^{\infty}      r^{-\frac{d_*}{2(k^\prime+1)}} dr\right)\notag\\
  \leq&O(n^{-(k^\prime+1)/d_*})+  n^{-1/2} \mathrm{polylog}(n)  \left(\tilde{h}_0(R_n)+\tilde{C}R_n^{d_*/2}(p_{k^\prime+1}(R_n))^{ \frac{d_*}{2(k^\prime+1)}}        \delta_n^{-\frac{d_*}{2(k^\prime+1)}+1} \right)\notag\\
    \leq&O(n^{-(k^\prime+1)/d_*})+  n^{-1/2}  \left(\mathrm{polylog}(n) +\mathrm{polylog}(n)         \delta_n^{-\frac{d_*}{2(k^\prime+1)}+1} \right)\notag\\
        =&O(n^{-(k^\prime+1)/d_*})+  n^{-1/2}  \left(\mathrm{polylog}(n) +\mathrm{polylog}(n)         n^{ \frac{1}{2 }-\frac{k^\prime+1}{d_*}} \right)\notag\,,
   \end{align}
  where $\tilde{C}>0$ denotes a constant that changes line by line and $\mathrm{polylog}(n)$ denotes a polynomial in $\log(n)$ (potentially different in each instance). Combining \eqref{eq:mean_bound_smooth_start}, \eqref{eq:mean_bound_smooth_envelope_term}, and \eqref{eq:mean_bound_smooth_covering_term} we obtain the claimed bound.   
  \end{proof}

\bibliography{GenMcidarmidUnboundedSupport}
\end{document}